%% file: main.tex
\documentclass[10pt,journal,compsoc]{IEEEtran}
\usepackage{cite}
\usepackage[pdftex]{graphicx}
\usepackage{amsmath}
\usepackage{algorithmic}
\usepackage{array}
\usepackage[caption=false,font=footnotesize,labelfont=sf,textfont=sf]{subfig}
\usepackage{url}
\usepackage[backref=section,
colorlinks=false,
citecolor=blue,
linkcolor=red,
anchorcolor=red,
unicode=true,
urlcolor=blue]{hyperref}
\usepackage{soul}
\usepackage{color}
\usepackage{parskip}
\usepackage{bm}
\usepackage{bbm}
\usepackage{mathtools}
\usepackage{amsthm}
\usepackage{amssymb}
\usepackage{amsfonts}
\usepackage{babel}
\usepackage{booktabs}
\usepackage{multirow}
\usepackage{enumitem}
\usepackage{pdfpages}
\usepackage{algorithm}
\renewcommand{\algorithmicrequire}{\textbf{Input:}}
\renewcommand{\algorithmicensure}{\textbf{Output:}}

\makeatletter
\theoremstyle{remark}
\newtheorem*{rem*}{\protect\remarkname}
\theoremstyle{plain}
\newtheorem{thm}{\protect\theoremname}
\theoremstyle{plain}

\theoremstyle{plain}
\newtheorem{lem}[thm]{\protect\lemmaname}
\theoremstyle{plain}

\theoremstyle{definition}

\theoremstyle{plain}
\newtheorem{prop}[thm]{\protect\propositionname}
\theoremstyle{definition}

\theoremstyle{definition}

\providecommand{\assumptionname}{Assumption}
\providecommand{\corollaryname}{Corollary}
\providecommand{\definitionname}{Definition}
\providecommand{\lemmaname}{Lemma}
\providecommand{\propositionname}{Proposition}
\providecommand{\remarkname}{Remark}
\providecommand{\theoremname}{Theorem}
\providecommand{\conditionname}{Condition}
\providecommand{\examplename}{Example}
\newcommand{\ci}[1]{\tiny{$\pm #1$}}

\makeatletter
\newcommand{\removelatexerror}{\let\@latex@error\@gobble}

\begin{document}
\title{
Knockoffs-SPR: Clean Sample Selection in Learning with Noisy Labels}
\author{
Yikai~Wang,
Yanwei~Fu,
and Xinwei~Sun.
\IEEEcompsocitemizethanks{
\IEEEcompsocthanksitem Yikai Wang and Yanwei Fu contribute equally.
\IEEEcompsocthanksitem Xinwei Sun is the corresponding author.
\IEEEcompsocthanksitem Yikai Wang, Yanwei Fu and Xinwei Sun are with the School of Data Science and MOE Frontiers Center
for Brain Science, Fudan University, Shanghai 200437, China, and also
with Fudan ISTBI–ZJNU Algorithm Centre for Brain-inspired Intelligence,
Zhejiang Normal University, Jinhua, Zhejiang 321017, China. E-mail: \{yikaiwang19, yanweifu, sunxinwei\}@fudan.edu.cn
}
}

\markboth{Journal of \LaTeX\ Class Files,~Vol.~14, No.~8, August~2015}%
{Shell \MakeLowercase{\textit{et al.}}: Bare Demo of IEEEtran.cls for Computer Society Journals}

\IEEEtitleabstractindextext{%
\begin{abstract}
A noisy training set usually leads to the degradation of the generalization and robustness of neural networks.
In this paper, we propose  a novel theoretically guaranteed clean sample selection framework for learning with noisy labels.
Specifically, we first present a Scalable Penalized Regression (\emph{SPR}) method, to model the linear relation between network features and one-hot labels. In SPR,  the clean data are identified by the zero mean-shift parameters solved in the regression model. We theoretically show that SPR can recover clean data under some conditions.
Under general scenarios, the conditions may be no longer satisfied; and some noisy data are falsely selected as clean data.
To solve this problem, 
we propose a data-adaptive method for Scalable Penalized Regression  with Knockoff filters  (\emph{Knockoffs-SPR}), which is provable to control the False-Selection-Rate (FSR) in the selected clean data. To improve the efficiency, we further present a split algorithm that divides the whole training set into small pieces that can be solved in parallel to make the framework scalable to large datasets. 
While Knockoffs-SPR can be regarded as a sample selection module for a standard supervised training pipeline, we further combine it with a semi-supervised algorithm to exploit the support of noisy data as unlabeled data.
Experimental results on several benchmark datasets and real-world noisy datasets show the effectiveness of our framework and validate the theoretical results of Knockoffs-SPR.
Our code and pre-trained models are available at \url{https://github.com/Yikai-Wang/Knockoffs-SPR}.
\end{abstract}
}
\maketitle
\IEEEdisplaynontitleabstractindextext
\IEEEpeerreviewmaketitle

\input{chap/intro.tex}
\input{chap/clean_sample.tex}

\input{chap/fdr_control.tex}

\input{chap/learning.tex}
\input{chap/related.tex}

\input{chap/experiment.tex}
\input{chap/conclusion.tex}

\bibliographystyle{IEEEtran}
\bibliography{spr.bib}

\begin{IEEEbiography}[{\includegraphics[width=1in,height=1.25in,clip,keepaspectratio]{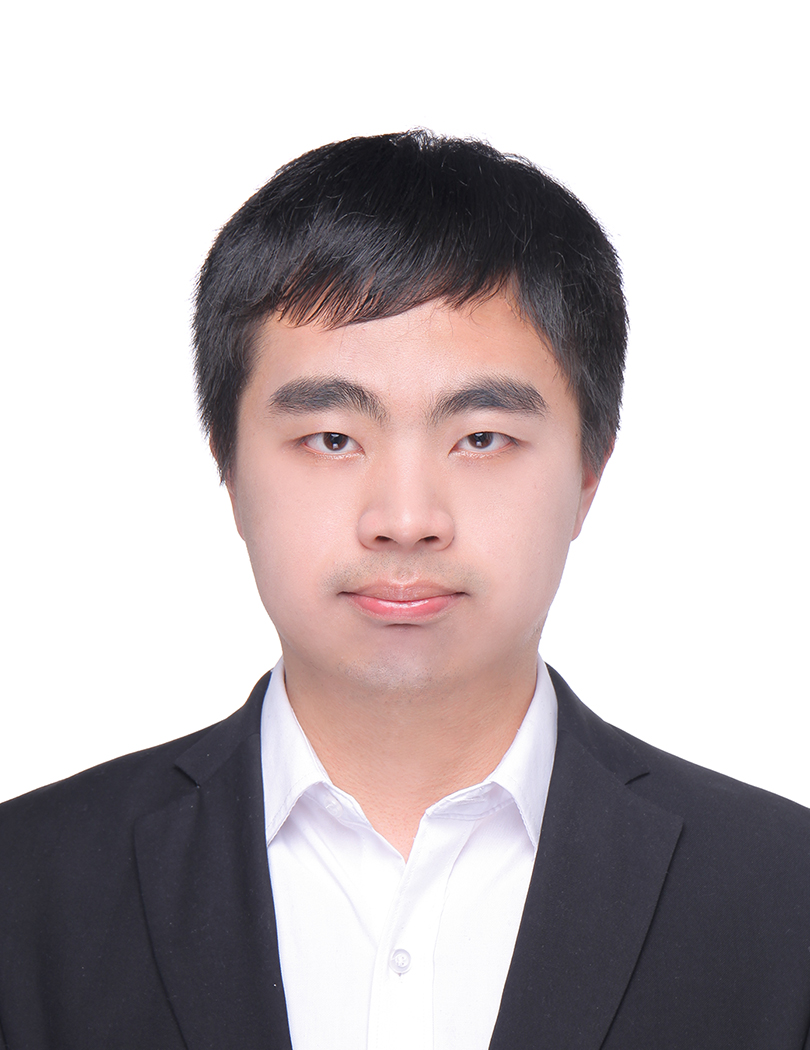}}]{Yikai Wang}
is a PhD candidate at the School of Data Science, Fudan University, under the supervision of Prof. Yanwei Fu. 
He received a Bachelor's degree in mathematics from the School of Mathematical Sciences, Fudan University, in 2019.
He published 1 IEEE TPAMI paper, 2 CVPR and 1 ICCV papers.
His current research interests include statistical machine learning and foundation models, with a focus on sparsity for sample selection.
\end{IEEEbiography}

\begin{IEEEbiography}[{\includegraphics[width=1in,height=1.25in,clip,keepaspectratio]{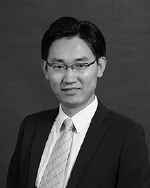}}]{Yanwei Fu} received the MEng degree
from the Department of Computer Science and
Technology, Nanjing University, China, in 2011, and
the PhD degree from the Queen Mary University of
London, in 2014. He held a post-doctoral position
with Disney Research, Pittsburgh, PA, from 2015 to
2016. He is currently a  professor with Fudan
University. He was appointed as the professor of
Special Appointment (Eastern Scholar) with Shanghai
Institutions of Higher Learning. His work has led
to many awards, including the IEEE ICME 2019 best
paper. He published more than 100 journal/conference papers including IEEE
Transactions on Pattern Analysis and Machine Intelligence, IEEE Transactions
on Multimedia, ECCV, and CVPR. His research interests are one-shot learning,
and learning-based 3D reconstruction
\end{IEEEbiography}

\begin{IEEEbiography}[{\includegraphics[width=1in,height=1.25in,clip,keepaspectratio]{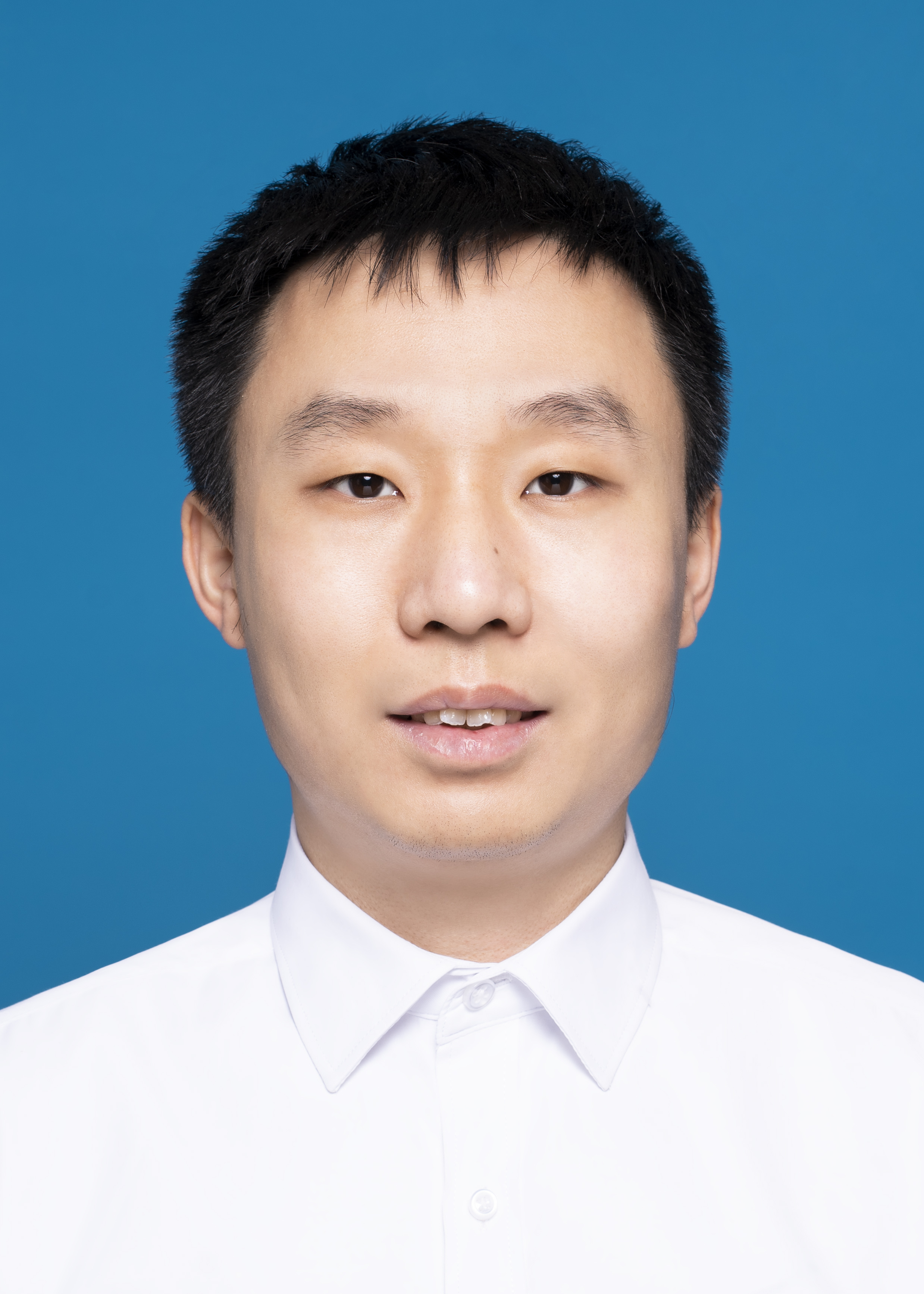}}]{Xinwei Sun}
is currently an assistant professor at the School of Data Science, Fudan University. He received his Ph.D. in the school of mathematical sciences, at Peking University in 2018. His research interests mainly focus on high-dimensional statistics and causal inference, with their applications in machine learning and medical imaging.
\end{IEEEbiography}

\input{appendix}

\end{document}

%% file: chap/intro.tex
\IEEEraisesectionheading{\section{Introduction}\label{sec:introduction}}
\begin{figure*}
\centering
\includegraphics[width=1\linewidth]{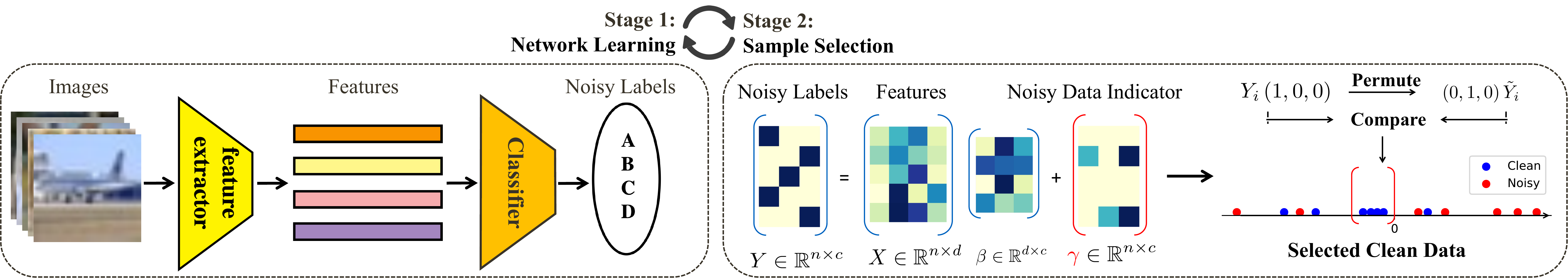}
\caption{Knockoffs-SPR runs a cycle between network learning and sample selection, where clean data are selected via the comparison of the mean-shift parameters between its original label and permuted label with a false-selection-rate control .
}
\label{fig:framework}
\end{figure*}
\IEEEPARstart{D}{eep} learning has achieved remarkable success on many  supervised learning tasks trained by millions of labeled training data. 
The performance of deep models heavily relies on the quality of label annotation since neural networks are susceptible to noisy labels and even can easily memorize randomly labeled annotations~\cite{zhang2017understanding}. Such noisy labels can lead to the degradation of the generalization and robustness of such models. Critically, 
it is expensive and difficult to obtain precise labels in many real-world scenarios,  thus exposing a realistic challenge for supervised deep models to learn with noisy data. 

There are many previous efforts in tackling this challenge by making the models robust to noisy data, such as modifying the network architectures ~\cite{xiao2015learning,goldberger2017training,chen2015webly,han2018masking} or loss functions~\cite{ghosh2017robust, zhang2018generalized, wang2019symmetric, lyu2020curriculum}.
This paper addresses the challenge by directly selecting clean samples. Inspired by the dynamic sample selection methods \cite{song2019selfie, lyu2020curriculum, han2018co, jiang2018mentornet,chen2019understanding,shen2019learning,yu2019does,nguyen2020self}, we construct 
a ``virtuous'' cycle between sample selection and network training:  
the selected clean samples can improve the network training; 
and on the other hand, the improved network  offers benefits in selecting clean data. 
As this cycle evolves, the performance can be improved. To well establish this, a key question remains: \emph{how to effectively differentiate  clean data from noisy ones?}

\textbf{Preliminary}. 
Typical principles in existing works \cite{song2019selfie, lyu2020curriculum, han2018co, jiang2018mentornet,chen2019understanding,shen2019learning,yu2019does,nguyen2020self} to differentiate clean data from noisy data
include large loss~\cite{han2018co}, inconsistent prediction~\cite{zhou2021robust}, and irregular feature representation~\cite{wu2020topological}. The former two principles identify irregular behaviors in the label space, while the last one analyzes the instance representations of the same class in the feature space.
In this paper, we propose unifying the label and feature space by making the linear relationship,
\begin{equation}
 \bm{y}_i=\bm{x}_i^{\top}\bm{\beta}+\bm{\varepsilon}, \label{eq:linear-equation}
\end{equation}
between feature-label pair ($\bm{x}_i\in\mathbb{R}^{p}$: feature vector; $\bm{y}_i\in\mathbb{R}^{c}$: one-hot label vector) of data $i$.
We also have the fixed (unknown) coefficient matrix
$\bm{\beta}\in\mathbb{R}^{p\times c}$, and random noise $\bm{\varepsilon}\in\mathbb{R}^{c}$. Essentially, the linear relationship here is an ideal approximation, as the networks are trained to minimize the divergence between a (soft-max) linear projection of the feature and a one-hot label vector. For a well-trained network, the output prediction of clean data is expected to be as similar to a one-hot vector as possible, while the entropy of the output of noisy data should be large.
Thus if the underlying linear relation is well-approximated without soft-max operation, the corresponding data is likely to be clean.
In contrast, the feature-label pair of noisy data may not be approximated well by the linear model. 

The simplest way to measure the fitting goodness of the linear model is to check the prediction error, or residual, $\bm{r}_i=\bm{y}_i-\bm{x}_i^{\top}\hat{\bm{\beta}}$, where $\hat{\bm{\beta}}$ is the estimate of $\bm{\beta}$.
The larger $\Vert\bm{r}_i\Vert$ indicates a larger error and thus more possibility for instance $i$ to be noisy data. Many methods have been proposed to test whether $\bm{r}_i$ is non-zero. Particularly, we highlight
the classical statistical leave-one-out approach~\cite{weisberg1985applied} that computes the studentized residual as,
\begin{equation}
\bm{t}_{i}=\frac{\bm{y}_{i}-\bm{x}_{i}^{\top} \hat{\bm{\beta}}_{-i}}{\hat{\sigma}_{-i}\left(1+\bm{x}_{i}^{\top}\left(\bm{X}_{-i}^{\top} \bm{X}_{-i}\right)^{-1} \bm{x}_{i}\right)^{1 / 2}},
\end{equation}
where $\hat{\sigma}$ is the scale estimate and the subscript $-i$ indicates estimates given the  $n-1$ data, leaving out the $i$-th data for testing.
Equivalently, the linear regression model can be re-formulated into explicitly representing the residual,
\begin{equation}
\bm{Y}=\bm{X}\bm{\beta}+\bm{\gamma}+\bm{\varepsilon},\quad\varepsilon_{i,j}\sim \mathcal{N}(0, \sigma^2),
\label{eq:lrip}
\end{equation}
by introducing a mean-shift parameter $\bm{\gamma}$ as in~\cite{she2011outlier} with the
feature $\bm{X}\in\mathbb{R}^{n\times p}$, and label $\bm{Y}\in\mathbb{R}^{n\times c}$ paired and stacked by rows. For
each row of $\bm{\gamma}\in\mathbb{R}^{n\times c}$, $\bm{\gamma}_i$ represents the predict residual of the $i$-th data.
This formulation has been widely studied in different research topics, including economics~\cite{neyman1948consistent, kiefer1956consistency, basu2011elimination, moreira2008maximum}, robust regression~\cite{she2011outlier,fan2018partial}, statistical ranking~\cite{fu2015robust}, face recognition~\cite{wright2009robust}, semi-supervised few-shot learning~\cite{wang2020instance,wang2021trust}, and Bayesian preference learning~\cite{simpson2020scalable}, to name a few.
This formulation is differently focused on the specific research tasks.
For example, for the robust regression problem~\cite{she2011outlier,fan2018partial}, the target is to get a robust estimate $\hat{\bm{\beta}}$ against the influence of $\bm{\gamma}$.
Here for sampling clean data from noisy labels, we are interested in recovering zero rows of $\bm{\gamma}$, as these elements correspond to clean data.

\textbf{SPR}~\cite{wang2022scalable}. 
To this end, from the statistical perspective, our conference report~\cite{wang2022scalable} starts from Eq.~\eqref{eq:lrip} to build up a sample selection framework, dubbed \emph{Scalable Penalized Regression} (SPR). With a sparse penalty $P(\gamma;\lambda)$ on $\bm{\gamma}$, the SPR obtains a regularization solution path of $\bm{\gamma}(\lambda)$ by evolving  $\lambda$  from $\infty$ to 0. Then it identifies those samples that are earlier (or at larger $\lambda$) selected to be non-zeros as noisy data and those later selected as clean data, with a manually specified ratio of selected data. Under the irrepresentable condition~\cite{wainwright2009sharp, zhao2006model}, the SPR enjoys model selection consistency in the sense that it can recover the set of noisy data. By feeding only clean data into next-round training, the trained network is less corrupted by the noisy data and hence performs well empirically. 

\textbf{Knockoffs-SPR}. However, the irrepresentable condition 
demands the prior of the ground-truth noisy set, which is not accessible in practice. 
Thus we cannot know whether SPR is theoretically guaranteed in practice, and when this condition fails, the trained network with SPR may be still corrupted by a large proportion of noisy data, leading to performance degradation as empirically verified in our experiments.
To amend this problem, we provide a data-adaptive sample selection algorithm, in order to well control the expected rate of noisy data in the selected data under the desired level $q$, \emph{e.g.}, $q = 0.05$. As the goal is to identify clean data for the next-round training, we term this rate as the \emph{False-Selection-Rate} (FSR). The FSR is the expected rate of the type-II error in sparse regression, as non-zero elements correspond to the noisy data. 
Our method to achieve the FSR control is inspired by the ideas of knockoffs in Statistics, 
which is a recently developed framework for variable selection \cite{barber2015controlling, dai2016knockoff, barber2019knockoff, cao2021control}.
The knockoffs aims at selecting non-null variables and controlling the False-Discovery-Rate (FDR), by taking as negative controls
knockoff features $\tilde{\bm{X}}$, which are constructed as a fake copy for the original features $\bm{X}$. As the FDR corresponds to the expectation of the type-I error rate in sparse regression, the vanilla knockoffs cannot be directly applied to sample selection, since FSR is the expected rate of the type-II error and there is no theoretical guarantee in knockoffs to control the type-II error. To achieve the FSR control, we propose Knockoffs-SPR, which turns to construct the knockoff labels $\tilde{\bm{Y}}$ via permutation for the original label $\bm{Y}$, and incorporates it into a data-partition strategy for FSR control.

Formally, we repurpose the knockoffs in Statistics in our SPR method; and propose a novel data-adaptive sample selection algorithm, dubbed Knockoffs-SPR. It extends SPR in controlling the ratio of noisy data among the selected clean data. With this property, Knockoffs-SPR ensures that the clean pattern is dominant in the data and hence leads to better network training. Specifically, we partition the whole noisy training set into two random subsets and apply the Knockoffs-SPR to two subsets separately. For each time, we use one subset to estimate $\bm{\beta}$ and the other to select the clean data by comparing between the solution paths of $\bm{\gamma}(\lambda)$ and $\tilde{\bm{\gamma}}(\lambda)$ that respectively obtained via regression on noisy labels and the permuted labels. With such a decoupled structure between $\bm{\beta}$ and $\bm{\gamma}$, we prove that the FSR can be controlled by any prescribed level. Compared with the original theory of SPR, our new theory enables us to effectively select clean data under general conditions. 

Together with network training, the whole framework is illustrated in Fig.~\ref{fig:framework} in which the sample selection and the network learning are well incorporated into each other.
Specifically, we run the network learning process and sample selection process iteratively and repeat this cycle until convergence. 
To incorporate Knockoffs-SPR into the end-to-end training pipeline of deep architecture, the simplest way is to directly solve Knockoffs-SPR for each training mini-batch or training epoch to select clean data. 
Solving Knockoffs-SPR for each mini-batch is efficient but suffers from the \textbf{identifiability} issue. 
The sample size in a mini-batch may be too small to distinguish clean patterns from noisy ones among all classes, especially for large datasets with small batch size.
Solving Knockoffs-SPR for the whole training set is powerful but suffers from the \textbf{complexity} issue, leading to an unacceptable computation cost. 
To resolve these two problems, we  strike a balance between complexity and identifiability by proposing a splitting strategy that divides the whole data into small pieces such that each piece is class-balanced with the proper sample size. 
The sample size of each piece is small enough to be solved efficiently and large enough to distinguish clean patterns from noisy ones.
Then Knockoffs-SPR runs on each piece in parallel, making it scalable to large datasets.

As the removed noisy data still contain useful information for network training, we adopt the semi-supervised training pipeline with CutMix~\cite{yun2019cutmix} where the noisy data are utilized as unlabeled data. 
We conduct extensive experiments to validate the effectiveness of our framework on several benchmark datasets and real-world noisy datasets. The results show the efficacy of our Knockoffs-SPR algorithm.

\textbf{Contributions}. 
Our contributions are as follows:
\begin{itemize}[leftmargin=*,itemsep=0pt,topsep=0pt,parsep=0pt]
\item \textbf{Ideologically}, we propose to control the False-Selection-Rate in selecting clean data, under general scenarios.
\item \textbf{Methodologically}, we propose Knockoffs-SPR, a data-adaptive method to control the FSR. 
\item \textbf{Theoretically}, we prove that the Knockoffs-SPR can control the FSR under any desired level. 
\item \textbf{Algorithmically}, we propose a splitting algorithm for better sample selection with balanced identifiability and complexity to scale up to large datasets.
\item \textbf{Experimentally}, we demonstrate the effectiveness and efficiency of our method on several benchmark datasets and real-world noisy datasets. 
\end{itemize}

\textbf{Extensions}.
Our conference version of this work, SPR, was published in~\cite{wang2022scalable}.
Compared with SPR~\cite{wang2022scalable}, we have the following extensions.
\begin{itemize}[leftmargin=*,itemsep=0pt,topsep=0pt,parsep=0pt]
\item We identify the limitation of the SPR and consider the FSR control in selecting clean data. 
\item We propose a new framework: Knockoffs-SPR which is effective in selecting clean data under general scenarios, theoretically and empirically. 
\item We apply our method on Clothing1M and achieve better results than compared baselines. 
\end{itemize}

\textbf{Logistics}.
The rest of this paper is organized as follows:
\begin{itemize}[leftmargin=*,itemsep=0pt,topsep=0pt,parsep=0pt]
\item In Section~\ref{sec:spr}, we introduce our SPR algorithm with its noisy set recovery theory.

\item In Section~\ref{sec:Knockoffs-SPR}, the Knockoffs-SPR algorithm is introduced with its FSR control theorem.

\item In Section~\ref{sec:learning}, several strategies are proposed to well incorporate the Knockoffs-SPR with the network training. 

\item In Section~\ref{sec:related}, connections are made between our proposed works and several previous works.

\item In Section~\ref{sec:experiments}, we conduct experiments on several synthetic and real-world noisy datasets, with further empirical analysis on Knockoffs-SPR.

\item Section~\ref{sec:conclusion} concludes this paper. 
\end{itemize}

%% file: chap/clean_sample.tex
\section{\label{sec:spr}Clean Sample Selection}
\subsection{Problem Setup}
We are given a dataset of image-label pairs $\{(\mathrm{img}_i,y_i)\}_{i=1}^n$, where the noisy label $y_i$ is corrupted from the ground-truth label $y_i^*$.
The ground-truth label $y_i^*$ and the corruption process are unknown.
Our target is to learn a model $f(\cdot)$ such that it can recognize the true class $y_i^*$ from the image $\mathrm{img}_i$, \emph{i.e.}, $f(\mathrm{img}_i)=y_i^*$, after training on the noisy label~$y_i$.

In this paper, we adopt deep neural networks as the recognition model and divide the $f(\cdot)$ into $\textrm{fc}(g(\cdot))$ where $g(\cdot)$ is the deep model for feature extraction and $\textrm{fc}(\cdot)$ is the final fully-connected layer for classification.
For each input image $\mathrm{img}_i$, the feature extractor $g(\cdot)$ is used to encode the feature $\bm{x}_i\coloneqq g(\mathrm{img}_i)$. Then the fully-connected layer is used to output the score vector $\hat{\bm{y}}_i=\textrm{fc}(\bm{x}_i)$ which indicates the chance it belongs to each class and the prediction is provided with $\hat{y}_i=\mathrm{argmax}(\hat{\bm{y}}_i)$.

As the training data contain many noisy labels, simply training from all the data leads to severe degradation of generalization and robustness.
Intuitively, if we could identify the clean labels from the noisy training set, and train the network with the clean data, we can reduce the influence of noisy labels and achieve better performance and robustness of the model.
To achieve this, we propose a sample selection algorithm to identify the clean data in the noisy training set with theoretical guarantees.

\textbf{Notation}. In this paper, we will use $a$ to represent scalar, $\bm{a}$ to represent a vector, $\bm{A}$ to represent a matrix, and $\mathcal{A}$ to represent a set.
We will annotate $a^*$ to denote the ground-truth value of $a$.
We use $\Vert\cdot\Vert_{\mathrm{F}}$ to denote the Frobenius norm.

\subsection{\label{subsec:spr}Clean Sample Selection via Penalized Regression}
Motivated by the leave-one-out approach for outlier detection, we introduce an explicit noisy data indicator $\bm{\gamma}_i$ for each data and assume a linear relation between extracted feature $\bm{x}_i$ and one-hot label $\bm{y}_i$ with noisy data indicator as,
\begin{equation}
\label{eq:model}
\bm{y}_{i}=\bm{x}_{i}^{\top}\bm{\beta}+\bm{\gamma}_i+\bm{\varepsilon}_i,   
\end{equation}
where $\bm{y}_i\in\mathbb{R}^{c}$ is one-hot vector for $c$-class task; and $\bm{x}_i\in\mathbb{R}^p, \bm{\beta}\in\mathbb{R}^{p\times c},\bm{\gamma}_i\in\mathbb{R}^{c},\bm{\varepsilon}_i\in\mathbb{R}^{c}$.
The noisy data indicator $\bm{\gamma}_i$ can be regarded as the correction of the linear prediction.
For clean data, $\bm{y}_i\sim\mathcal{N}(\bm{x}_{i}^{\top}\bm{\beta}^*, \sigma^2\bm{I}_c)$ with $\bm{\gamma}_i^*=0$, and for noisy data $\bm{y}_i^*=\bm{y}_i-\bm{\gamma}_i^*\sim\mathcal{N}(\bm{x}_{i}^{\top}\bm{\beta}^*, \sigma^2)$. We denote $\mathcal{C}:=\{i: \bm{\gamma}^*_i = 0\}$ as the ground-truth clean set.

To select clean data for training, 
we propose \emph{Scalable Penalized Regression} (SPR), designed as 
the following sparse learning paradigm,
\begin{equation}
\label{eq:original-problem}
\underset{\bm{\beta},\bm{\gamma}}{\mathrm{argmin}}  \frac{1}{2}\left\Vert \bm{Y}-\bm{X}\bm{\beta}-\bm{\gamma}\right\Vert _{\mathrm{F}}^{2}+P(\bm{\gamma};\lambda),
\end{equation}
where we have the matrix formulation  $\bm{X} \in \mathbb{R}^{n \times p}$, and $ \bm{Y} \in \mathbb{R}^{n \times c}$  of  $\{\bm{x}_i,\bm{y}_i\}_{i=1}^n$; and $P(\cdot;\lambda)$ is a row-wise sparse penalty with coefficient parameter $\lambda$. So we have $P(\bm{\gamma};\lambda)=\sum_{j=1}^n P(\bm{\gamma}_i;\lambda)$, \emph{e.g.}, group-lasso sparsity with $P(\bm{\gamma};\lambda)=\lambda \sum_i \Vert \bm{\gamma}_i \Vert_2$.

\begin{figure}
\centering
\includegraphics[width=0.7\linewidth]{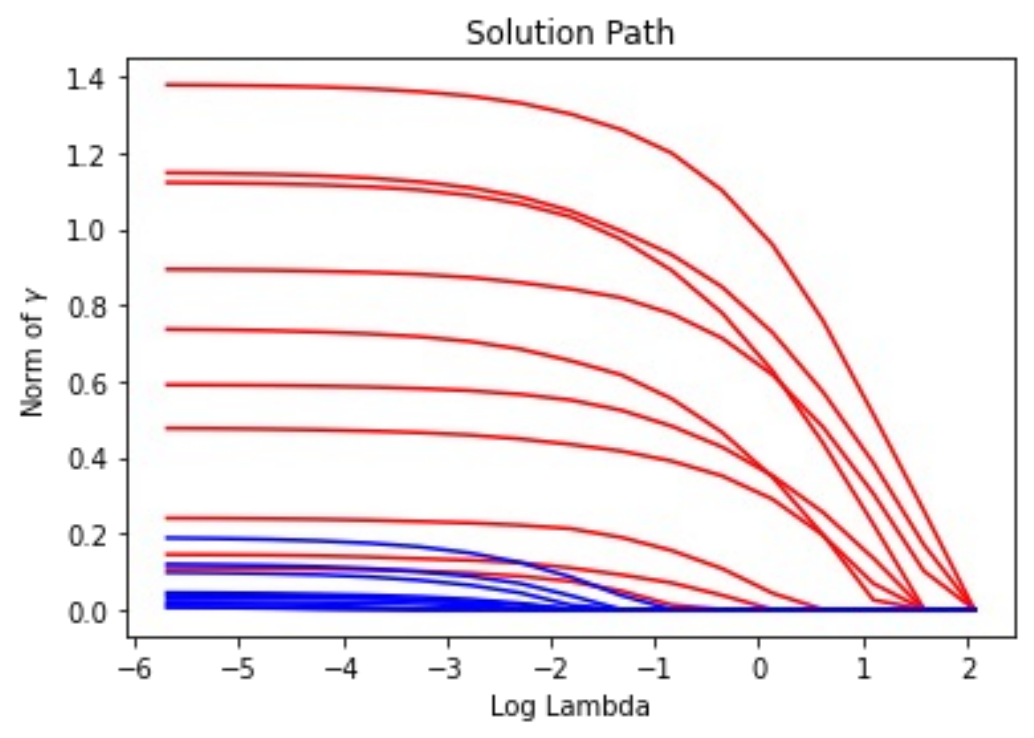}
\caption{Solution Path of SPR. Red lines indicate noisy data while blue lines indicate clean data.
As $\lambda$ decreases, the $\bm{\gamma}_i$ gradually solved with non-zero values.
}
\label{fig:spr}
\end{figure}

To estimate $\mathcal{C}$, we only need to solve $\bm{\gamma}$ with no need to estimate $\bm{\beta}$.
Thus to simplify the optimization, we substitute the Ordinary Least Squares (OLS) estimate for $\bm{\beta}$  with $\bm{\gamma}$ fixed into Eq.~\eqref{eq:original-problem}. 
To ensure that $\hat{\bm{\beta}}$ is identifiable, we apply PCA on $\bm{X}$ to make $p \ll n$ so that the $\bm{X}$ has full-column rank. 
Denote $\tilde{\bm{X}}=\bm{I}-\bm{X}\left(\bm{X}^{\top}\bm{X}\right)^{\dagger}\bm{X}^{\top},\tilde{\bm{Y}}=\tilde{\bm{X}}\bm{Y}$, the Eq.~\eqref{eq:original-problem} is transformed into
\begin{equation}
\underset{\bm{\gamma}}{\mathrm{argmin}}\frac{1}{2}\left\Vert \tilde{\bm{Y}}-\tilde{\bm{X}}\bm{\gamma}\right\Vert _{\mathrm{F}}^{2}+P(\bm{\gamma};\lambda), 
\label{eq:spr}
\end{equation}
which is a standard sparse linear regression for $\bm{\gamma}$. 
Note that in practice we can hardly choose a proper $\lambda$ that works well in all scenarios.
Furthermore, from the equivalence between the penalized regression problem and Huber's M-estimate, the solution of $\bm{\gamma}$ is returned with soft-thresholding.
Thus it is not worth  finding the precise solution of a single $\bm{\gamma}$.
Instead, we use a block-wise descent algorithm~\cite{simon2013blockwise} to solve $\bm{\gamma}$ with a list of $\lambda$s and generate the solution path.
As $\lambda$ changes from $\infty$ to $0$, the influence of sparse penalty decreases, and $\bm{\gamma}_i$ are gradually solved with non-zero values, in other words, selected by the model, as visualized in Fig.~\ref{fig:spr}.
Since earlier selected instance is more possible to be noisy, we rank all samples in the descendent order of their selecting time defined as:  
\begin{equation}
Z_{i}=\sup\left\{ \lambda:\bm{\gamma}_{i}\left(\lambda\right)\neq0\right\}.
\label{eq:select}
\end{equation}
A large $Z_i$ indicates earlier selected $\bm{\gamma}_i$. 
Then the top samples are identified as noisy data and the others are selected as clean data.
In the absence of knowledge about the clean ratio, we select 50\% of the data as clean data in practice.

\subsection{The Theory of Noisy Set Recovery in SPR}
The SPR enjoys theoretical guarantees that the noisy data set can be fully recovered with high probability, under the irrepresentable condition~\cite{zhao2006model}.
Formally, consider the vectorized version of Eq.~\eqref{eq:spr}:
\begin{equation}
\underset{\Vec{\bm{\gamma}}}{\mathrm{argmin}}\frac{1}{2}\left\Vert \Vec{\bm{y}}-\mathring{\bm{X}}\Vec{\bm{\gamma}}\right\Vert _{2}^{2}+\lambda\left\Vert \Vec{\bm{\gamma}}\right\Vert _{1}, 
\label{eq:vec-spr}
\end{equation}
where $\Vec{\bm{y}},\Vec{\bm{\gamma}}$ is vectorized from $\bm{Y},\bm{\gamma}$ in Eq.~\eqref{eq:spr}; $\mathring{\bm{X}}=I_{c}\otimes\tilde{\bm{X}}$ with $\otimes$ denoting the Kronecker product operator and $\bm{I}_c$ is the $c\times c$ identity matrix. 
Denote $\mathcal{S}:=\mathrm{supp}(\Vec{\bm{\gamma}}^*)$, which is the noisy set $\mathcal{C}^c$. 
We further denote $\mathring{\bm{X}}_{\mathcal{S}}$ (resp. $\mathring{\bm{X}}_{\mathcal{S}^c}$) as the column vectors of $\mathring{\bm{X}}$ whose indexes are in $\mathcal{S}$ (resp. $\mathcal{S}^c$) and $\mu_{\mathring{\bm{X}}}=\max_{i\in \mathcal{S}^c}\Vert\mathring{\bm{X}}\Vert_2^2$. 
Then we have
\begin{thm}[Noisy set recovery]
\label{thm:original}
Assume that:\newline 
\textbf{\emph{C1, Restricted eigenvalue:}} $\lambda_{\min}(\mathring{\bm{X}}_{\mathcal{S}}^{\top}\mathring{\bm{X}}_{\mathcal{S}})=C_{\min} >0$;\newline 
\textbf{\emph{C2, Irrepresentability:}} there exists a
$\eta \in (0,1]$, such that  $\Vert\mathring{\bm{X}}_{\mathcal{S}^c}^{\top}\mathring{\bm{X}}_{\mathcal{S}}(\mathring{\bm{X}}_{\mathcal{S}}^{\top}\mathring{\bm{X}}_{\mathcal{S}})^{-1}\Vert_\infty \leq 1-\eta$;\newline
\textbf{\emph{C3, Large error:}}
$\quad\Vec{\bm{\gamma}}^*_{\min}\coloneqq \min_{i\in \mathcal{S}}|\Vec{\bm{\gamma}}^*_{i}|>h(\lambda,\eta,\mathring{\bm{X}},\Vec{\bm{\gamma}}^*)$;\newline
where $\Vert\bm{A}\Vert_\infty\coloneqq \max_i \sum_j |A_{i,j}|$, and 
$h(\lambda,\eta,\mathring{\bm{X}},\Vec{\bm{\gamma}}^*)=\lambda\eta/\sqrt{C_{\min}\mu_{\mathring{\bm{X}}}}+\lambda \Vert(\mathring{\bm{X}}_{\mathcal{S}}^{\top}\mathring{\bm{X}}_{\mathcal{S}})^{-1}\mathrm{sign}(\Vec{\bm{\gamma}}_{\mathcal{S}}^*)\Vert_\infty$.\newline 
Let $\lambda\geq \frac{2\sigma \sqrt{\mu_{\mathring{\bm{X}}}}}{\eta}\sqrt{\log cn}$. 
Then with probability greater than $1-2(cn)^{-1}$, model Eq.~\eqref{eq:vec-spr} has a unique solution $\hat{\Vec{\bm{\gamma}}}$ such that: 1) If C1 and C2 hold, $\hat{\mathcal{C}^c}\subseteq \mathcal{C}^c$;2) If C1, C2 and C3 hold, $\hat{\mathcal{C}^c}= \mathcal{C}^c$.
\end{thm}
We present the proof in the appendix, following the treatment in~\cite{wainwright2009sharp, xu2021evaluating}.
In this theorem,
C1 is necessary to get a unique solution, and in our case is mostly satisfied with the natural assumption that  the clean data is the majority in the training data.
If C2 holds, the estimated noisy data is the subset of truly noisy data.
This condition is the key to ensuring the success of SPR, which requires divergence between clean and noisy data such that we cannot represent clean data with noisy data.
If C3 further holds, the estimated noisy data is exactly all the truly noisy data.
C3 requires the error measured by $\bm{\gamma}_i$ is large enough to be identified from random noise.

%% file: chap/fdr_control.tex
\section{\label{sec:Knockoffs-SPR}Controlled Clean Sample Selection}

In the last section, we stop the solution path at $\lambda$ such that 50\% samples are selected as clean data. If this happens to be the rate of clean data, Theorem~\ref{thm:original} shows that our SPR can identify the clean dataset $\mathcal{C}$ under the irrepresentable condition. However, the irrepresentable condition and the information of the ground-truth clean set $\mathcal{C}$ are practically unknown, making this theory hard to be used in the real life. Particularly, with $|\mathcal{C}^c|$ unknown, the algorithm can stop at an improper time such that the noise rate of the selected clean data $\hat{\mathcal{C}}$ can be still high, making the next-round trained model corrupted a lot by noisy patterns.

To resolve the problem of false selection in SPR
, we in this section propose a data-adaptive early stopping method for the solution path, that targets controlling the expected noise rate of the selected data dubbed as \emph{False-Selection-Rate} (FSR) under the desired level $q$ ($0 < q < 1$): 
\begin{equation}
\label{eq:FSR}
\mathrm{FSR}=\mathbb{E}\left[\frac{\#\left\{ j: j\not\in \mathcal{H}_0 \cap \hat{\mathcal{C}}\right\} }{\#\left\{ j:j\in\hat{\mathcal{C}}\right\} \lor1}\right],
\end{equation}
where $\hat{\mathcal{C}}=\{j:\hat{\gamma}_j= 0\}$ is the recovered clean set of $\bm{\gamma}$, and $\mathcal{H}_0: \bm{\gamma}^*_i = 0$ denotes the null hypothesis, \emph{i.e.}, the sample $i$ belonging to the clean dataset. Therefore, the FSR in Eq.~\eqref{eq:FSR} targets controlling the false rate among selected null hypotheses, which is also called the expected rate of the type-II error in hypothesis testing.

\subsection{Knockoffs-SPR}
In SPR, we compare statistics of different samples derived from the noisy data indicator $\gamma$.
While it is a straightforward method to assess the credibility of the annotated label, it can also be advantageous to explore an alternative perspective:
How plausible is the annotated label compared to a randomly permuted label?
This notion prompts us to develop a novel statistics that not only preserves the credibility compared with other data but also maintains its consistency in relation to a permuted label. 
Specifically, this statistics allows us to effectively control the FSR.

Formally, to achieve the FSR control, we propose the Knockoffs-SPR for clean sample selection.
Our method is inspired by knockoff methods~\cite{barber2015controlling, dai2016knockoff, xu2016false, barber2019knockoff, cao2021control} with the different focus that we target selecting clean labels via permutation instead of constructing knockoff features to select explanatory variables.
Specifically, under model~\eqref{eq:model} we permute the label for each data and construct the permutation $\tilde{\bm{y}}$.
Then model~\eqref{eq:model} can be solved for $\bm{y}$ and $\tilde{\bm{y}}$ to obtain the solution paths $\bm{\gamma}(\lambda)$ and $\tilde{\bm{\gamma}}(\lambda)$, respectively. 
We will show that this construction can pick up clean data from noisy ones, by comparing the selecting time (Eq.~\eqref{eq:select}) between $\bm{\gamma}(\lambda)$ and $\tilde{\bm{\gamma}}(\lambda)$ for each data. 
On the basis of this construction, we propose to partition the whole dataset into two disjoint parts, with one part for estimating 
$\bm{\beta}$
and the other for learning $\bm{\gamma}(\lambda)$ and $\tilde{\bm{\gamma}}(\lambda)$. We will show that the independent structure with such a  data partition enables us to construct the comparison statistics whose signs among alternative hypotheses (noisy data) are the independent Bernoulli processes, which is crucial for FSR control.

Specifically, we split the whole data $\mathcal{D}$ into $\mathcal{D}_1:=(\bm{X}_1,\bm{Y}_1)$ and $\mathcal{D}_2:=(\bm{X}_2,\bm{Y}_2)$ with $n_i:=|\mathcal{D}_i|$, and implement Knockoffs-SPR on both $\mathcal{D}_1$ and $\mathcal{D}_2$. In the following, we only introduce the procedure on $\mathcal{D}_2$, as the procedure for $\mathcal{D}_1$ shares the same spirit. The procedure is composed of three steps: \textbf{i)} \emph{estimate $\bm{\beta}$ on $\mathcal{D}_1$}; \textbf{ii)} \emph{estimate $(\bm{\gamma}(\lambda), \tilde{\bm{\gamma}}(\lambda))$ on $\mathcal{D}_2$}; and  \textbf{iii)} \emph{construct the comparison statistics and selection filters}. We leave detailed discussions for each step in Sec.~\ref{sec.discuss}.

\textbf{Step i): Estimating $\bm{\beta}$ on $\mathcal{D}_1$.} 
Our target is to provide an estimate of $\bm{\beta}$ that is independent of $\mathcal{D}_2$.
The simplest strategy is to use the standard OLS estimate to obtain $\hat{\bm{\beta}}_1$.
However, this estimate may not be accurate since it is corrupted by noisy samples. For this consideration, we first run SPR on $\mathcal{D}_1$ to get clean data and then solve $\bm{\beta}$ via OLS on the estimated clean data.

\textbf{Step ii): Estimating $\left(\bm{\gamma}(\lambda),\tilde{\bm{\gamma}}(\lambda)\right)$ on $\mathcal{D}_2$.} 
After obtaining the solution 
$\hat{\bm{\beta}}_1$ on $\mathcal{D}_1$
, we learn the $\bm{\gamma}(\lambda)$ on $\mathcal{D}_2$: 
\begin{equation}
\label{eq:spr-d2}
\frac{1}{2}\left\Vert \bm{Y}_2-\bm{X}_2\hat{\bm{\beta}}_1-\bm{\gamma}_{2}\right\Vert_{\mathrm{F}}^{2}+P(\bm{\gamma}_{2};\lambda).
\end{equation}
For each one-hot encoded vector $\bm{y}_{2,j}$, we randomly permute the position of 1 and obtain another one-hot vector $\tilde{\bm{y}}_{2,j}\neq\bm{y}_{2,j}$. 
For clean data $j$, the $\tilde{\bm{y}}_{2,j}$ turns to be a noisy label; while for noisy data, the $\tilde{\bm{y}}_{2,j}$ is switched to another noisy label with probability $\frac{c-2}{c-1}$ or clean label with probability $\frac{1}{c-1}$. 
After obtaining the permuted matrix as $\tilde{\bm{Y}}_2$, we learn the solution paths $\left(\bm{\gamma}_{2}(\lambda), \tilde{\bm{\gamma}}_2(\lambda)\right)$ using the same algorithm as SPR via: 
\begin{equation}
\label{eq:spr-permute}
\begin{cases}
\frac{1}{2}\left\Vert \bm{Y}_2-\bm{X}_2\hat{\bm{\beta}}_1-\bm{\gamma}_2\right\Vert _{\mathrm{F}}^{2}+\sum_j P(\bm{\gamma}_{2,j};\lambda),\\
\frac{1}{2}\left\Vert \tilde{\bm{Y}}_2-\bm{X}_2\hat{\bm{\beta}}_1-\tilde{\bm{\gamma}}_2\right\Vert _{\mathrm{F}}^{2}+\sum_j P(\tilde{\bm{\gamma}}_{2,j};\lambda).
\end{cases}
\end{equation}

\textbf{Step iii): Comparison statistics and selection filters.} After obtaining the solution path $\left(\bm{\gamma}_{2}(\lambda), \tilde{\bm{\gamma}}_2(\lambda)\right)$, we define sample significance scores  with respect to $\bm{y}_{2,j}$ and $\tilde{\bm{y}}_{2,j}$ of  each $j$ respectively, as the selection time: $Z_j := \sup\{\lambda: \Vert \bm{\gamma}_{2,j}(\lambda) \Vert_2 \neq 0\}$ and $\tilde{Z}_j := \sup\{\lambda: \Vert \tilde{\bm{\gamma}}_{2,j} (\lambda) \Vert_2 \neq 0\}$. With $Z_j, \tilde{Z}_j$, we define the $W_j$ as: 
\begin{equation}
\label{eq:w}
    W_j := Z_j \cdot \mathrm{sign}(Z_j - \tilde{Z}_j). 
\end{equation}

Based on these statistics, we define a data-dependent threshold $T$ as 
\begin{equation}
\label{eq:T}
T=\max\left\{ t>0:\frac{1 + \#\left\{ j:0<W_{j}\leq t\right\} }{\#\left\{ j:-t\leq W_{j}<0\right\} \lor 1}\leq q\right\},
\end{equation}
or \ensuremath{T=0} if this set is empty,
where $q$ is the pre-defined upper bound.
Our algorithm will select the clean subset identified by 
\begin{equation}
\label{eq:clean-set}
\mathcal{C}_2\coloneqq\{j:-T\leq W_j <0\}.
\end{equation}
In practice, after calculating $\{W_j\}_{j=1}^{n_2}$,  we arrange them in descending order based on their magnitudes $|W_j|$.
For a given $q$, 
our process begins with an initial $T$ set to the maximum magnitude, i.e., $T = \max_j |W_j|$. We then assess whether the condition in Eq.~\eqref{eq:T} is met. If it's not satisfied, we remove the largest value, $\max_j |W_j|$, from consideration and redefine $T$ as the maximum magnitude among the remaining values in $\{|W_j|\}$.
This iterative procedure continues until we identify a feasible $T$ such that the inequality holds. 
This ensures that we have found the highest attainable $T$ that controls FSR in the desired level $q$.

Empirically, $T$ may be equal to $0$ if the threshold $q$ is sufficiently small. In this regard, no clean data are selected, which is meaningless.
Therefore, we start with a small $q$ and iteratively increase $q$ and calculate $T$, until an attainable $T$ such that $T > 0$ to bound the FSR as small as possible.
In practice, when the FSR cannot be bounded by $q=50\%$, we will end the selection and simply select half of the most possible clean examples via $\{W_j\}$. 

The whole procedure of Knockoffs-SPR is shown in Algorithm~\ref{alg:Knockoffs-SPR-single}.
The Knockoffs-SPR construct randomly permuted labels as controls to select truly clean samples while controlling the FSR, \emph{i.e.}, the expected ratio of noisy samples among selected clean samples. Intuitively, it selects the clean sample with its annotated label being more credible than its copy, \emph{i.e.}, a randomly permuted label. To measure the credibility, we compute both $Z_j$ and $\tilde{Z}_j$. Smaller $Z$ indicates larger credibility of the label $Y$. As a result, we take those with small negatives of $W := Z \cdot \mathrm{sign}(Z - \tilde{Z})$ as clean samples. Based on $W$, we then compute a data-dependent threshold $T$ to control the FSR.

The rigorous proof of FSR relies on an independent structure between the magnitude $|W|$ that is determined by $\hat{\bm{\beta}}$ and the $\mathrm{sign}(W)$ that is determined by $(\bm{\gamma},\tilde{\bm{\gamma}})$. This motivates us to partition the full dataset into two independent parts, with the 1st step estimating $\hat{\bm{\beta}}$ and the 2nd step calculating $(\bm{\gamma},\tilde{\bm{\gamma}})$. 

In summary, the Knockoff-SPR needs to partition the dataset into two parts and calculate $W$ and $T$. Specifically, in step i, we break down the estimation of both $\bm{\beta}$ and $\bm{\gamma}$ with data partition to introduce an independent structure. We then use the first part $\mathcal{D}_1$ to calculate $\hat{\bm{\beta}}$; moving to step ii, we construct the randomly permuted label $\tilde{\bm{Y}}$ (line 2), as knockoffs copy and calculate $W$ using the second part $\mathcal{D}_2$ (line 3). 
We then obtain $T$ by iteratively increasing $q$ until we can select the clean sample set (line 4-7). 
Finally, we obtain the clean sample set $\mathcal{C}_2 := \{i: - T \leq W_i(t) <0 \}$. 
Note that we empirically implement $\tilde{\bm{Y}}$ as the most confident permuted label by the network, see Sec.~\ref{sec:learning-knockoff} for details.

\begin{figure}
\begin{algorithm}[H]
\caption{\label{alg:Knockoffs-SPR-single}Knockoffs-SPR}
\begin{algorithmic}[1]
\REQUIRE subsets $\mathcal{D}_1$ and $\mathcal{D}_2$, $q=0.02$, $T=0$.
\ENSURE Clean set $\mathcal{C}_2$ of $\mathcal{D}_2$.
\STATE Use $\mathcal{D}_1$ to  get $\hat{\bm{\beta}}(\mathcal{D}_1)$;
\STATE Generate most-confident permuted label $\tilde{\bm{Y}}_2$ from $\mathcal{D}_2$;
\STATE Solve Eq.~\eqref{eq:spr-permute} for $\mathcal{D}_2$ and generate $\{W_i\}$ by Eq.~\eqref{eq:w};
\WHILE{$q<0.5$ and $T=0$}
\STATE Compute $T$ by Eq.~\eqref{eq:T};
\STATE $q = q + 0.02$;
\ENDWHILE
\STATE Construct  $\mathcal{C}_2$ via samples in Eq.~\eqref{eq:clean-set};
\RETURN $\mathcal{C}_2$.
\end{algorithmic}
\end{algorithm}
\end{figure}

\subsection{\label{sec:discussion-Knockoffs-SPR}Statistical Analysis about Knockoffs-SPR}
\label{sec.discuss}
In this part, we present the motivations and intuitions of each step in Knockoffs-SPR.

\textbf{Data Partition}.
Knockoffs-SPR partitions the dataset $\mathcal{D}$ into two subset $\mathcal{D}_1$ and $\mathcal{D}_2$.
This step decomposes the dependency of the estimate of $\bm{\beta}$ and $\bm{\gamma}$ in that we use $\mathcal{D}_1$/$\mathcal{D}_2$ to estimate $\bm{\beta}$/$\bm{\gamma}$, respectively. 
Then $\hat{\bm{\beta}}(\mathcal{D}_1)$ is independent of $\hat{\bm{\gamma}}(\mathcal{D}_2)$ if $\mathcal{D}_1$ and $\mathcal{D}_2$ are disjoint.
This construction induces the  independent estimates $\{\mathrm{sign}(W_j)\}_{j=1}^{n_2}$, making it provable for FSR control on $\mathcal{D}_2$.

\textbf{Permutation}.
As we discussed in step ii, when the original label is clean, its permuted label will be a noisy label. On the other hand, if the original label is noisy, its permuted label changes to clean with probability $\frac{1}{c-1}$ and noisy with probability $\frac{c-2}{c-1}$, where $c$ denotes the number of classes. Note that $\bm{\gamma}$ of noisy data is often selected earlier than that of clean data in the solution path. This implies larger $Z$ values for noisy data than those for clean data. As a result, according to the definition of $W$, a clean sample will ideally have a small negative of $W:=Z \cdot \mathrm{sign}(Z-\tilde{Z})$, where $Z$ and $\tilde{Z}$ respectively correspond to the clean label and noisy label. In contrast for a noisy sample, the $W$ tends to have a large magnitude and has approximately equal probability to be positive or negative. Such a different behavior of $W$ between clean and noisy data can help us to identify clean samples from noisy ones.

\textbf{Asymmetric comparison statistics $W$}.
The classical way to define comparison statistics is in 
a symmetric manner, i.e., $W_j\coloneqq Z_j \vee \tilde{Z}_j\cdot\mathrm{sign}(Z_j-\tilde{Z}_j)$.
In this way, a clean sample with a noisy permuted label tends to have a large $|W_j|$, as we expect the noisy label to have a large $\tilde{Z}_j$.
However, this is against our target as we only require clean samples to have small magnitude. Thus we design asymmetric statistics that only consider the magnitude of the original labels.

To see the asymmetric behavior of $W$ for noisy and clean data, we consider the Karush–Kuhn–Tucker (KKT) conditions of Eq.~\eqref{eq:spr-permute} with respect to $(\bm{\gamma}_{2,j},\tilde{\bm{\gamma}}_{2,j})$
\begin{subequations}
\label{eq:kkt}
\begin{align}
\bm{\gamma}_{2,j}+\frac{\partial P(\bm{\gamma}_{2,j};\lambda)}{\partial\bm{\gamma}_{2,j}}=\bm{x}_{2,j}^{\top}(\bm{\beta}^*-\hat{\bm{\beta}}_1)+\bm{\gamma}_{2,j}^{*}+\bm{\varepsilon}_{2,j},\label{eq:kkt-1} \\
\tilde{\bm{\gamma}}_{2,j}+\frac{\partial P(\tilde{\bm{\gamma}}_{2,j};\lambda)}{\partial\tilde{\bm{\gamma}}_{2,j}}=\bm{x}_{2,j}^{\top}(\bm{\beta}^{*}-\hat{\bm{\beta}}_1)+\tilde{\bm{\gamma}}^{*}_{2,j}+\tilde{\bm{\varepsilon}}_{2,j}\label{eq:kkt-2},
\end{align}
\end{subequations}
where $\bm{\varepsilon}_{2,j} \sim_{i.i.d} \tilde{\bm{\varepsilon}}_{2,j}$, $\Vert\bm{\gamma}_{2,j}^{*}\Vert = \Vert\tilde{\bm{\gamma}}^*_{2,j}\Vert$ if both $\bm{y}_{2,j}$ and $\tilde{\bm{y}}_{2,j}$ are noisy, and $P(\bm{\gamma}_{2,j};\lambda):=\lambda \Vert \bm{\gamma}_{2,j}\Vert$ as an example. By conditioning on $\hat{\bm{\beta}}_1$ and denoting $\bm{a}_j:=\bm{x}_{2,j}^{\top}(\bm{\beta}^*-\hat{\bm{\beta}}_1)$, we have that 
\begin{equation}
    P(W_j>0) = P(\Vert\bm{a}_j+\bm{\gamma}_{2,j}^{*}+\bm{\varepsilon}_{2,j}\Vert>\Vert\bm{a}_j+\tilde{\bm{\gamma}}_{2,j}^{*}+\tilde{\bm{\varepsilon}}_{(2),j}\Vert).
\end{equation}
Then it can be seen that if $j$ is clean, we have $\bm{\gamma}_{2,j}^{*} = 0$. Then $Z_j$ tends to be small and besides, it is probable to have $Z_j < \tilde{Z}_j$ if $\hat{\bm{\beta}}_1$ can estimate $\beta^*$ well. As a result, $W_j$ tends to be a small negative. On the other hand, if $j$ is noisy, then $Z_j$ tends to be large for $\bm{\gamma}_j$ to account for the noisy pattern, and besides, it has equal probability between $Z_j < \tilde{Z}_j$ and $Z_j \geq \tilde{Z}_j$ when $\tilde{\bm{y}}_{2,j}$ is switched to another noisy label, with probability $\frac{c-2}{c-1}$. So $W_j$ tends to have a large value and besides, 
\begin{align}
    & P(W_j > 0) = P(W_j>0|\tilde{\bm{y}}_{2,j} \text{ is noisy})P(\tilde{\bm{y}}_{2,j} \text{ is noisy}) \nonumber \\
    & + P(W_j>0|\tilde{\bm{y}}_{2,j} \text{ is clean})P(\tilde{\bm{y}}_{2,j} \text{ is clean}) 
= \frac{1}{2} \cdot \frac{c-2}{c-1} \nonumber \\
& + P(W_i>0|\tilde{\bm{y}}_{2,j} \text{ is clean}) \cdot \frac{1}{c-1}, 
    \label{eq:w-prob}
\end{align}
which falls in the interval of $ \left[\frac{c-2}{c-1} \cdot \frac{1}{2},\frac{c}{c-1} \cdot \frac{1}{2}\right]$. That is to say, $P(W_j>0) \approx \frac{1}{2}$. In this regard, the clean data corresponds to small negatives of $W$ in the ideal case, which can help to discriminate noisy data with large $W$ with almost equal probability to be positive or negative.

\begin{rem*}
    For noisy $\bm{y}_{2,j}$, we have $P(W_j>0|\tilde{\bm{y}}_{2,j} \text{ is noisy}) = 1/2$ by assuming $\Vert\bm{\gamma}_{2,j}^{*}\Vert = \Vert\tilde{\bm{\gamma}}^*_{2,j}\Vert$. However, it may not hold in practice when $\bm{y}_{2,j}$ corresponds to the noisy pattern that has been learned by the model. In this regard, it may have $|\bm{\gamma}_{2,j}^{*}| < |\tilde{\bm{\gamma}}^*_{2,j}|$ for a randomly permuted label $\tilde{\bm{y}}_{2,j}$. To resolve this problem, we instead set the permutation label as the most confident candidate of the model, please refer to Sec.~\ref{sec:learning-knockoff} for details. Besides, if $\hat{\bm{\beta}}_1$ can accurately estimate $\bm{\beta}^*$, according to KKT conditions in Eq.~\eqref{eq:kkt}, we have $P(W_j>0) < 1/2$. That is $W_j$ tends to be negative for the clean data, which is beneficial for clean sample selection. 
\end{rem*}

\textbf{Data-adaptive threshold}.
The proposed data-adaptive threshold $T$ is directly designed to control the FSR.
Specifically, the FSR defined in Eq.~\eqref{eq:FSR} is equivalent to 
\begin{equation}
\label{eq:FSR-in-spr}
\mathrm{FSR}(t)=\mathbb{E}\left[\frac{\#\left\{ j:\bm{\gamma}_{j}\neq0\textrm{ and }-t\leq W_j<0\right\} }{\#\left\{ j:-t\leq W_j<0\right\} \lor1}\right],
\end{equation}
where the denominator denotes the number of selected clean data according to Eq.~\eqref{eq:clean-set} and the nominator denotes the number of falsely selected noisy data. This form of Eq.~\eqref{eq:FSR-in-spr} can be further decomposed into,
{\small
\begin{align}
& \mathbb{E}\left[\frac{\#\left\{\bm{\gamma}_{j}\neq0, \ -t\leq W_j<0\right\}}{1+\#\left\{\bm{\gamma}_{j}\neq0, \ 0< W_j\leq t\right\}} \cdot \frac{1+\#\left\{\bm{\gamma}_{j}\neq0, \ 0< W_j\leq t\right\}}{\#\left\{-t\leq W_j<0\right\} \lor1}\right] \nonumber \\
& \leq \mathbb{E}\left[\frac{\#\left\{\bm{\gamma}_{j}\neq0, \ -t\leq W_j<0\right\}}{1+\#\left\{\bm{\gamma}_{j}\neq0, \ 0< W_j\leq t\right\}} \frac{1+\#\left\{0< W_j\leq t\right\}}{\#\left\{-t\leq W_j<0\right\} \lor1}\right] \nonumber \\
& \leq \mathbb{E}\left[\frac{\#\left\{\bm{\gamma}_{j}\neq0, \ -t\leq W_j<0\right\}}{1+\#\left\{\bm{\gamma}_{j}\neq0, \ 0< W_j\leq t\right\}} q\right], \label{eq:fsr-decompose}
\end{align}
}

where the last inequality comes from the definition of $T$ in Eq.~\eqref{eq:T}. To control the FSR, it suffices to bound $\mathbb{E}\left[\frac{\#\left\{\bm{\gamma}_{j}\neq0, \ -t\leq W_j<0\right\}}{1+\#\left\{\bm{\gamma}_{j}\neq0, \ 0< W_j\leq t\right\}}\right]$. Roughly speaking, this term means the number of negative $W$ to the number of positive $W$, among noisy data. Since $W$ for noisy data has approximately equal probability to be positive/negative as mentioned earlier, intuitively we have this term $\approx \frac{1}{2}$. Formally, we construct a martingale process of $\mathbbm{1}(W_i>0)$ among noisy data.
We leave these details in the appendix.

\subsection{FSR Control of Knockoffs-SPR}
Our target is to show that $\mathrm{FSR}\leq q$ with our data-adaptive threshold $T$ in Eq.~\eqref{eq:T}.
Our main result is as follows:
\begin{thm}[FSR control]
Under model~\eqref{eq:model}, for $c$-class classification task, and for all $0<q\leq1$, the solution of Knockoffs-SPR holds
\begin{equation}
\mathrm{FSR}(T)\leq q
\end{equation}
with the threshold $T$ for two subsets defined respectively as
 \begin{equation*}
 T_i=\max\left\{ {\scriptstyle t\in\mathcal{W}} 
 :\frac{1+\#\left\{ j:0<W_{j}\leq t\right\} }{\#\left\{ j:-t\leq W_{j}<0\right\} \lor 1}\leq  \frac{c-2+2\kappa_i}{2(c-2\kappa_i)}q\right\}.
 \end{equation*}
where 
$\kappa_i\coloneqq\min_{\bm{\gamma}_j^*\neq 0,j\in\mathcal{D}_i}\mathbb{P}(W_j>0 | \tilde{\bm{\gamma}}_j^*=0 )\in (0,1)$.

For multi-class setting $c>2$, we also have well-defined $T_i$ as 
\begin{equation*}
\label{eq:T-multi}
T_i=\max\left\{ t\in\mathcal{W}:\frac{1+\#\left\{ j:0<W_{j}\leq t\right\} }{\#\left\{ j:-t\leq W_{j}<0\right\} \lor 1}\leq \frac{c-2}{2c}q\right\}.
\end{equation*}
\end{thm}
We present the proof in the appendix.
The coefficient $1/2$ comes from the subset-partition strategy that we run Knockoffs-SPR on two $\mathcal{D}_1$ and $\mathcal{D}_2$, and the term $\frac{c-2+2\kappa}{(c-2\kappa)}$ comes from the upper-bound of the first part in Eq.~\eqref{eq:fsr-decompose}.
This theorem tells us that FSR can be controlled by the given threshold $q$ using the procedure of Knockoffs-SPR. 
Compared to SPR, this procedure is more practical and useful in real-world experiments and we demonstrate its utility in Sec.~\ref{sec:selection-quality} for more details.

\begin{figure}
\begin{algorithm}[H]
\caption{\label{alg:Knockoffs-SPR}Knockoffs-SPR on full training set}
\begin{algorithmic}[1]
\REQUIRE $\mathcal{D}=\{(\bm{x}_i, \bm{y}_i)\}_{i=1}^n$, class size $N$ sample size $m$, (Optional) \textit{clean set} $\mathcal{C}$.
\ENSURE  \textit{clean set} $\mathcal{C}$. 
\STATE Determine classes of group from Eq.~\eqref{eq:simialrity} and Eq.~\eqref{eq:prototype};
\STATE Uniformly construct \textit{pieces} with samples of size $N\times m$;
\FOR{each \textit{piece}}
\STATE Partition \textit{piece} into \textit{sub-pieces} $\mathcal{A}$ and $\mathcal{B}$ ;
\STATE Run  Algorithm~\ref{alg:Knockoffs-SPR-single} ($\mathcal{B}, \mathcal{A}$) on $\mathcal{A}$ to get \textit{clean-set-}$\mathcal{A}$;
\STATE Run  Algorithm~\ref{alg:Knockoffs-SPR-single} ($\mathcal{A}, \mathcal{B}$) on $\mathcal{B}$ to get \textit{clean-set-}$\mathcal{B}$;
\STATE Concat \textit{clean-set-}$\mathcal{A}$ and \textit{clean-set-}$\mathcal{B}$ to get \textit{clean-set-piece};
\ENDFOR
\STATE Concat \textit{clean-set-piece}s to get \textit{clean set} $\mathcal{C}$;
\RETURN \textit{clean set} $\mathcal{C}$.
\end{algorithmic}
\end{algorithm}
\end{figure}

%% file: chap/learning.tex
\section{\label{sec:learning}Learning with Knockoffs-SPR}
In this section, we introduce how to incorporate Knockoffs-SPR into the training of neural networks.
We first introduce several implementation details of Knockoffs-SPR, then we introduce a splitting algorithm that makes Knockoffs-SPR scalable to large-scale datasets. 
Finally, we discuss some training strategies to better utilize the selected clean data.

\subsection{Knockoffs-SPR in Practice}
\label{sec:learning-knockoff}
We introduce several strategies to improve FSR control and the power of selecting clean samples, which are inspired by different behaviors of $W$ between noisy and clean samples. Ideally, for a clean sample $j$, $W_j$ is expected to be a small negative; if $j$ is noisy data, $W_j$ tends to be large and is approximately 50\% to be positive or negative, as shown in Eq.~\eqref{eq:w-prob}. To achieve these properties for better clean sample selection, the following strategies are proposed, in the procedure of \emph{feature extractor}, \emph{data-preprocessing}, \emph{label permutation strategy}, \emph{estimating $\bm{\beta}$ on $\mathcal{D}_1$}, and \emph{clean data identification in Eq.~\eqref{eq:T},~\eqref{eq:clean-set}}.

\textbf{Feature Extractor.} A good feature extractor is essential for clean sample selection algorithms. In our experiments, we adopt the self-supervised training method SimSiam~\cite{chen2021exploring} to pre-train the feature extractor, to make $\bm{X}$ well encode the information of the training data in the early stages.

\textbf{Data Preprocessing.} We implement PCA on the features extracted by neural network for dimension reduction. This can make $\bm{X}$ of full rank, which ensures the identifiability of $\hat{\bm{\beta}}$ in SPR. Besides, such a low dimensionality can make the model estimate $\bm{\beta}$ more accurately. According to the KKT conditions Eq.~\eqref{eq:kkt}, we have that $W_j$ of clean data $j$ tends to be negative with small magnitudes. In this regard, the model can have better power of clean sample selection, \emph{i.e.}, selecting more clean samples while controlling FSR. 

\textbf{Label Permutation Strategy.} Instead of the random permutation strategy, our Knckoff-SPR permutes the label as the most-confident candidate provided by the model, for FSR consideration especially when the noise rate is high or some noisy pattern is dominant in the data. Specifically, if the pattern of some noisy label $\bm{y}_{2,j}$ is learned by the model, then $\bm{\gamma}_{2,j}^*$ may have a smaller magnitude than that of $\tilde{\bm{\gamma}}_{2,j}^*$ for a randomly permuted label $\tilde{\bm{y}}_{2,j}$ that may not be learned by the model, violating $P(W_j>0|\tilde{\bm{y}}_{2,j})=1/2$ and hence $P(W_j>0)\approx 1/2$ in practice. The most confident permutation alleviate this problem, as the most confident label $\tilde{\bm{y}}_{2,j}$ can naturally have a small magnitude of $\tilde{\bm{\gamma}}_{2,j}^*$. 

\textbf{Estimating $\bm{\beta}$ on $\mathcal{D}_1$.} We implement SPR as the first step to learn $\bm{\beta}$ on $\mathcal{D}_1$. Compared to vanilla OLS, the SPR can remove some noisy patterns from data, and hence can achieve an accurate estimate of $\bm{\beta}$. Similar to the \emph{data processing} step, such an accurate estimation can improve the power of selecting clean samples. 

\textbf{Clean data identification in Eq.~\eqref{eq:T},~\eqref{eq:clean-set}.} We calculate $T$ among $W$ for each class, and identify the clean subset for each class, to improve the power of clean data for each class. In practice, since some classes may be easier to learn than others, the $W_j$ for $j$ in these classes have smaller magnitudes. Therefore, data from these 
classes will take the main proportion if we calculate $T$ and identify $\mathcal{C}_2$ among all classes. With this design, the clean data are more balanced, which facilitates the training in the next epochs.

\subsection{Scalable to Large Dataset}

The computation cost of the sample selection algorithm increases with the growth of the training sample, making it not scalable to large datasets. To resolve this problem, we propose to split the total training set into many pieces, each of which contains a small portion of training categories with a small number of training data. With the splitting strategy, we can run the Knockoffs-SPR on several pieces in parallel and significantly reduce the running time. For the splitting strategy, we notice that the key to identifying clean data is leveraging different behavior in terms of the magnitude and the sign of $W$. Such a difference can be alleviated if the patterns from clean classes are similar to the noisy ones, which may lead to unsatisfactory recall/power of identifying the clean set. 
This motivates us to group  similar categories together, to facilitate the discrimination of clean data from noisy ones.

Formally speaking, we define the similarity between the class $i$ and $j$ as
\begin{equation}
\label{eq:simialrity}
s(i,j)=\bm{p}_i^{\top}\bm{p}_j,
\end{equation} 
where $\bm{p}$ represents the class prototype. To obtain $\bm{p}_i$ for the class $i$, we take the clean features $\bm{x}_i$ of each class extracted by the network along the training iteration, and average them to get the class prototype $\bm{p}_c$ after the current training epoch ends, as
\begin{equation}
\label{eq:prototype}
\bm{p}_c=\frac{\sum_{i=1}^n \bm{x}_i\mathbbm{1}(y_i=c,i\in \mathcal{C})}{\sum_{i=1}^n \mathbbm{1}(y_i=c,i\in \mathcal{C})},
\end{equation}
Then the most similar classes are grouped together. 
In the initialization step when the clean set has not been estimated yet, we simply use all the data to calculate the class prototypes.
In our experiments, each \emph{group} is designed to have 10 classes. 

For the instances in each group, we split the training data of each class in a balanced way such that each \emph{piece} contains the same number of instances for each class.
The number is determined to ensure that the clean pattern remains the majority in the piece, such that optimization can be done easily.
In practice, we select 75 training data from each class to construct the piece.
When the class proportion is imbalanced in the original dataset, we adopt the over-sampling strategy to sample the instance of each class with less training data multiple times to ensure that each training instance is selected once in some piece.
The pipeline of our splitting algorithm is described in Algorithm~\ref{alg:Knockoffs-SPR}.

\subsection{Network Learning with Knockoffs-SPR}

\begin{figure}
\begin{algorithm}[H]
\caption{\label{alg:pipeline}Training with Knockoffs-SPR.}
\begin{algorithmic}[1]
\renewcommand{\algorithmicrequire}{\textbf{Input:}}
\renewcommand{\algorithmicensure}{\textbf{Output:}}
\REQUIRE Noisy dataset $\{(\mathrm{img}_i, \bm{y}_i)\}_{i=1}^n$, probability $p$.
\ENSURE  Trained network.
\\ \textit{Initialization} : 
\STATE Self-supervised pre-trained backbone, an EMA model;
\STATE Get $\bm{x}$ from self-supervised pre-trained backbone;
\STATE \textit{clean set}: Run Algorithm~\ref{alg:Knockoffs-SPR} on $\{\bm{x}_i,\bm{y}_i\}_{i=1}^n$;
\\ \textit{Training Process}:
\FOR  {ep = 0 to max\_epochs}
\FOR{\emph{each mini-batch}}
\STATE Train with Eq.~\eqref{eq:loss-cutmix} or~\eqref{eq:loss-supervised} with chance $p$ and $1-p$;
\STATE Update features $\bm{x}$ visited in current mini-batch;
\STATE Update EMA model;
\ENDFOR
\STATE Run Algorithm~\ref{alg:Knockoffs-SPR} on $\{(\bm{x}_i, \bm{y}_i)\}_{i=1}^n$ to get \textit{clean set};
\ENDFOR
\RETURN Trained network.
\end{algorithmic}
\end{algorithm}
\end{figure}

When training with Knockoffs-SPR, we can further exploit the support of noisy data by incorporating Knockoffs-SPR with semi-supervised algorithms.
In this paper, we interpolate part of images between clean data and noisy data as in CutMix~\cite{yun2019cutmix},
\begin{subequations}
\label{eq:cutmix}
\begin{align}
\tilde{\mathrm{img}}&=\bm{M}\odot\mathrm{img}_{\mathrm{clean}}+(1-\bm{M})\odot\mathrm{img}_{\mathrm{noisy}}\\
\tilde{\bm{y}}&=\lambda \bm{y}_{\mathrm{clean}}+(1-\lambda)\bm{y}_{\mathrm{noisy}}
\end{align}
\end{subequations}
where $\bm{M}\in\{0,1\}^{W\times H}$ is a binary mask, $\odot$ is element-wise multiplication, $\lambda\sim\mathrm{Beta}(0.5,0.5)$ is the interpolation coefficient, and the clean and noisy data are identified by Knockoffs-SPR.
Then we train the network using the interpolated data using
\begin{equation}
\mathcal{L}\left(\tilde{\mathrm{img}},\tilde{\bm{y}}\right) = \mathcal{L}_{\mathrm{CE}}\left(\tilde{\mathrm{img}},\tilde{\bm{y}}\right),
\label{eq:loss-cutmix}
\end{equation} 
where $\mathcal{L}_{\mathrm{CE}}$ indicates the cross-entropy loss. Empirically, we could switch between this semi-supervised training with standard supervised training on estimated clean data.
\begin{equation}
\mathcal{L}\left(\mathrm{img}_i,\bm{y}_i\right) = \mathbbm{1}(i\in \mathcal{C})\cdot \mathcal{L}_{\mathrm{CE}}\left(\mathrm{img}_i,\bm{y}_i\right),
\label{eq:loss-supervised}
\end{equation} 
where $\mathbbm{1}(i\in \mathcal{C})$ is the indicator function, which means that only the cross-entropy loss of estimated clean data is used to calculate the loss.
We further store a model with EMA-updated weights.
Our full algorithm is illustrated in Algorithm~\ref{alg:pipeline}.
Neural networks trained with this pipeline enjoy powerful recognition capacity in several synthetic and real-world noisy datasets.

%% file: chap/related.tex
\section{\label{sec:related}Related Work}

Here we make the connections between our Knockoffs-SPR and previous research efforts.
\subsection{Learning with Noisy Labels}
The target of Learning with Noisy Labels (LNL) is to train a more robust model from the noisy dataset.
We can roughly categorize LNL algorithms into two groups: robust algorithm and noise detection.
A robust algorithm does not focus on specific noisy data but designs specific modules to ensure that networks can be well-trained even from the noisy datasets.
Methods following this direction includes constructing robust network~\cite{xiao2015learning, goldberger2017training,chen2015webly, han2018masking}, robust loss function~\cite{ghosh2017robust, zhang2018generalized, wang2019symmetric, lyu2020curriculum,10039708}  robust regularization~\cite{tanno2019learning,menon2020can, xia2021robust,zhou2021learning} against noisy labels.

The noise detection method aims to identify the noisy data and design specific strategies to deal with the noisy data, including down-weighting the importance in the loss function for the network training~\cite{thulasidasan2019combating}, re-labeling them to get correct labels~\cite{tanaka2018joint}, or regarding them as unlabeled data in the semi-supervised manner~\cite{Li2020DivideMix}, etc.

For the noise detection algorithm, noisy data are identified by some irregular patterns, including large error~\cite{shen2019learning}, gradient directions~\cite{ren2018learning}, disagreement within multiple networks~\cite{yu2019does}, inconsistency along the training path~\cite{zhou2021robust} and some spatial properties in the training data~\cite{wang2018iterative,lee2019robust,wu2020topological, dong2020extreme}.
Some algorithms~\cite{veit2017learning, ren2018learning} rely on the existence of an extra clean set to detect noisy data.

After detecting the clean data, the simplest strategy is to train the network using the clean data only or re-weight the data~\cite{patrini2017making} to eliminate the noise.
Some algorithms~\cite{Li2020DivideMix,arazo2019unsupervised} regard the detected noisy data as unlabeled data to fully exploit the distribution support of the training set in the semi-supervised learning manner.
There are also some studies of designing label-correction module~\cite{xiao2015learning,vahdat2017toward,veit2017learning,li2017learning,tanaka2018joint,yi2019probabilistic} to further pseudo-labeling the noisy data to train the network. Few of these approaches are designed from the statistical perspective with non-asymptotic guarantees, in terms of clean sample selection. In contrast, our Knockoffs-SPR can theoretically control the false-selected rate in selecting clean samples under general scenarios.

\subsection{Mean-Shift Parameter}
Mean-shift parameters or incidental parameters~\cite{neyman1948consistent} originally tackled to solve the robust estimation problem via penalized estimation~\cite{fan2010selective}
.
With a different focus on specific parameters, this formulation address wide attention in different research topics, including economics~\cite{neyman1948consistent, kiefer1956consistency, basu2011elimination, moreira2008maximum}, robust regression~\cite{she2011outlier,fan2018partial}, statistical ranking~\cite{fu2015robust}, face recognition~\cite{wright2009robust}, semi-supervised few-shot learning~\cite{wang2020instance,wang2021trust}, and Bayesian preference learning~\cite{simpson2020scalable}, to name a few.
Previous work usually uses this formulation to solve robust linear models, while in this paper we adopt this to select clean data and help the training of neural networks.
Furthermore, we design an FSR control module and a scalable sample selection algorithm based on mean-shift parameters with theoretical guarantees.

\begin{table*}
\caption{Test accuracies(\%) on several benchmark datasets with different settings.
\label{tab:main-results}}
\centering
\begin{tabular}{llccccccc}
\toprule 
\multirow{2}{*}{Dataset} & \multirow{2}{*}{Method} &  \multicolumn{4}{c}{Sym. Noise Rate} &  \multicolumn{3}{c}{Asy. Noise Rate}\tabularnewline
& & 0.2 & 0.4 & 0.6 & 0.8 & 0.2 & 0.3 & 0.4 \tabularnewline
\midrule
\multirow{17}{*}{CIFAR-10}
& Standard & $85.7\pm 0.5$ & $81.8\pm0.6$ & $73.7\pm1.1$ & $42.0\pm2.8$ & $88.0\pm0.3$ & $86.4\pm0.4$ & $84.9\pm0.7$ \tabularnewline
& Forgetting & $86.0\pm0.8$ & $82.1\pm0.7$ & $75.5\pm0.7$ & $41.3\pm3.3$ & $89.5\pm0.2$ & $88.2\pm0.1$ & $85.0\pm1.0$ \tabularnewline
& Bootstrap & $86.4\pm0.6$ & $82.5\pm0.1$ & $75.2\pm0.8$ & $42.1\pm3.3$ & $88.8\pm0.5$ & $87.5\pm0.5$ & $85.1\pm0.3$ \tabularnewline
& Forward & $85.7\pm0.4$ & $81.0\pm0.4$ & $73.3\pm1.1$ & $31.6\pm4.0$ & $88.5\pm0.4$ & $87.3\pm0.2$ & $85.3\pm0.6$ \tabularnewline
& Decoupling & $87.4\pm0.3$ & $83.3\pm0.4$ & $73.8\pm1.0$ & $36.0\pm3.2$ & $89.3\pm0.3$ & $88.1\pm0.4$ & $85.1\pm1.0$ \tabularnewline
& MentorNet & $88.1\pm0.3$ & $81.4\pm0.5$ & $70.4\pm1.1$ & $31.3\pm2.9$ & $86.3\pm0.4$ & $84.8\pm0.3$ & $78.7\pm0.4$ \tabularnewline
& Co-teaching & $89.2\pm0.3$ & $86.4\pm0.4$ & $79.0\pm0.2$ & $22.9\pm3.5$ & $90.0\pm0.2$ & $88.2\pm0.1$ & $78.4\pm0.7$ \tabularnewline
& Co-teaching+ & $89.8\pm0.2$ & $86.1\pm0.2$ & $74.0\pm0.2$ & $17.9\pm1.1$ & $89.4\pm0.2$ & $87.1\pm0.5$ & $71.3\pm0.8$ \tabularnewline
& IterNLD & $87.9\pm0.4$ & $83.7\pm0.4$ & $74.1\pm0.5$ & $38.0\pm1.9$ & $89.3\pm0.3$ & $88.8\pm0.5$ & $85.0\pm0.4$ \tabularnewline
& RoG & $89.2\pm0.3$ & $83.5\pm0.4$ & $77.9\pm0.6$ & $29.1\pm1.8$ & $89.6\pm0.4$ & $88.4\pm0.5$ & $86.2\pm0.6$ \tabularnewline
& PENCIL & $88.2\pm0.2$ & $86.6\pm0.3$ & $74.3\pm0.6$ & $45.3\pm1.4$ & $90.2\pm0.2$ & $88.3\pm0.2$ & $84.5\pm0.5$ \tabularnewline
& GCE & $88.7\pm0.3$ & $84.7\pm0.4$ & $76.1\pm0.3$ & $41.7\pm1.0$ & $88.1\pm0.3$ & $86.0\pm0.4$ & $81.4\pm0.6$ \tabularnewline 
& SL & $89.2\pm0.5$ & $85.3\pm0.7$ & $78.0\pm0.3$ & $44.4\pm1.1$ & $88.7\pm0.3$ & $86.3\pm0.1$ & $81.4\pm0.7$ \tabularnewline
& TopoFilter & $90.2\pm0.2$ & $87.2\pm0.4$ & $80.5\pm0.4$ & $45.7\pm1.0$ & $90.5\pm0.2$ & $89.7\pm0.3$ & $87.9\pm0.2$ \tabularnewline
\cline{2-9}\noalign{\vspace{0.5ex}}
& SPR & $92.0\pm0.1$ & $\bm{94.6\pm0.2}$ & $91.6\pm0.2$ & $80.5\pm0.6$ & $89.0\pm0.8$ & $90.3\pm0.8$ & $91.0\pm0.6$ \tabularnewline
& Knockoffs-SPR & $\bm{95.4\pm0.1}$ & $94.5\pm0.1$ & $\bm{93.3\pm0.1}$ & $\bm{84.6\pm0.8}$ & $\bm{95.1\pm0.1}$ & $\bm{94.5\pm0.2}$ & $\bm{93.6\pm0.2}$\tabularnewline
\midrule
\multirow{15}{*}{CIFAR-100}
& Standard & $56.5\pm0.7$ & $50.4\pm0.8$ & $38.7\pm1.0$ & $18.4\pm0.5$ & $57.3\pm0.7$ & $52.2\pm0.4$ & $42.3\pm0.7$ \tabularnewline
& Forgetting & $56.5\pm0.7$ & $50.6\pm0.9$ & $38.7\pm1.0$ & $18.4\pm0.4$ & $57.5\pm1.1$ & $52.4\pm0.8$ & $42.4\pm0.8$ \tabularnewline
& Bootstrap & $56.2\pm0.5$ & $50.8\pm0.6$ & $37.7\pm0.8$ & $19.0\pm0.6$ & $57.1\pm0.9$ & $53.0\pm0.9$ & $43.0\pm1.0$ \tabularnewline
& Forward & $56.4\pm0.4$ & $49.7\pm1.3$ & $38.0\pm1.5$ & $12.8\pm1.3$ & $56.8\pm1.0$ & $52.7\pm0.5$ & $42.0\pm1.0$ \tabularnewline
& Decoupling & $57.8\pm0.4$ & $49.9\pm1.0$ & $37.8\pm0.7$ & $17.0\pm0.7$ & $60.2\pm0.9$ & $54.9\pm0.1$ & $47.2\pm0.9$ \tabularnewline
& MentorNet & $62.9\pm1.2$ & $52.8\pm0.7$  & $36.0\pm1.5$ & $15.1\pm0.9$ & $62.3\pm1.3$ & $55.3\pm0.5$ & $44.4\pm1.6$ \tabularnewline
& Co-teaching & $64.8\pm0.2$ & $60.3\pm0.4$ & $46.8\pm0.7$ & $13.3\pm2.8$ & $63.6\pm0.4$ & $58.3\pm1.1$ & $48.9\pm0.8$\tabularnewline
& Co-teaching+ & $64.2\pm0.4$ & $53.1\pm0.2$ & $25.3\pm0.5$ & $10.1\pm1.2$ & $60.9\pm0.3$ & $56.8\pm0.5$ & $48.6\pm0.4$ \tabularnewline
& IterNLD & $57.9\pm0.4$ & $51.2\pm0.4$ & $38.1\pm0.9$ & $15.5\pm0.8$ & $58.1\pm0.4$ & $53.0\pm0.3$ & $43.5\pm0.8$ \tabularnewline
& RoG & $63.1\pm0.3$ & $58.2\pm0.5$ & $47.4\pm0.8$ & $20.0\pm0.9$ & $67.1\pm0.6$ & $65.6\pm0.4$ & $58.8\pm0.1$ \tabularnewline
& PENCIL & $64.9\pm0.3$ & $61.3\pm0.4$ & $46.6\pm0.7$ & $17.3\pm0.8$ & $67.5\pm0.5$ & $66.0\pm0.4$ & $61.9\pm0.4$ \tabularnewline
& GCE & $63.6\pm0.6$  & $59.8\pm0.5$ & $46.5\pm1.3$ & $17.0\pm1.1$ & $64.8\pm0.9$ & $61.4\pm1.1$ & $50.4\pm0.9$ \tabularnewline 
& SL & $62.1\pm0.4$  & $55.6\pm0.6$  & $42.7\pm0.8$  & $19.5\pm0.7$ & $59.2\pm0.6$ & $55.1\pm0.7$ & $44.8\pm0.1$\tabularnewline
& TopoFilter & $65.6\pm0.3$ & $62.0\pm0.6$ & $47.7\pm0.5$  & $20.7\pm1.2$ & $68.0\pm0.3$ & $66.7\pm0.6$ & $62.4\pm0.2$ \tabularnewline
\cline{2-9}\noalign{\vspace{0.5ex}}
& SPR & $72.5\pm0.2$ & $\bm{75.0\pm0.1}$ & $\bm{70.9\pm0.3}$ & $\bm{38.1\pm0.8}$ & $71.9\pm0.2$ & $72.4\pm0.3$ & $70.9\pm0.5$
\tabularnewline
& Knockoffs-SPR & $\bm{77.5\pm0.2}$ & $74.3\pm0.2$ & $67.8\pm0.4$ & $30.5\pm1.0$ & $\bm{77.3\pm0.4}$ & $\bm{76.3\pm0.3}$ & $\bm{73.9\pm0.6}$
\tabularnewline
\bottomrule
\end{tabular}
\end{table*}

\subsection{Knockoffs}
Knockoffs was first proposed in \cite{barber2015controlling} as a data-adaptive method to control FDR of variable selection in the sparse regression problem. This method was then extended to high-dimension regression \cite{barber2019knockoff, candes2018panning}, multi-task regression \cite{dai2016knockoff}, outlier detection \cite{xu2016false} and structural sparsity \cite{cao2021control}. The core of Knockoffs is to construct a fake copy of $\bm{X}$ as negative controls of original features, in order to select true positive features with FDR control. Our Knockoffs-SPR is inspired by but different from the classical knockoffs in the following aspects:  
\textbf{i)} Model Assumption:
The original knockoffs method investigates the model of $\bm{y}=\bm{x}^{\top}\bm{\beta}+\bm{\varepsilon}$, while Knockoffs-SPR study the model $\bm{y}=\bm{x}^{\top}\bm{\beta}+\bm{\gamma}+\bm{\varepsilon}$.
\textbf{ii)} Target:
The primary goal of the original knockoffs method is to select a subset of $\bm{\beta}$ or columns of $\bm{x}$ for feature selection, identifying patterns that genuinely impact the response variables. In contrast, Knockoffs-SPR focuses on sample selection, aiming to choose a subset of rows from the $\gamma$ matrix among all available training data.
\textbf{iii)} Focused Problem:
The original knockoff method aims to control false-positive selections of non-zero $\bm{\beta}$, akin to controlling type I errors. 
Knockoffs-SPR, on the other hand, targets the control of false-selections of zero $\bm{\gamma}$, analogous to controlling type II errors.
\textbf{iv)} Construction of Filters:
In the original knockoffs method, a knockoff copy $\tilde{\bm{X}}$ of the design matrix $\bm{X}$ is constructed to replicate its correlation structure. In the case of Knockoffs-SPR, a ``knockoff'' copy of the label matrix $\bm{Y}$ is created by permuting the indices of the $1$s in each row.
Equipped with a calibrated data-partitioning strategy, our method can control the FER under any desired level.

%% file: chap/experiment.tex
\section{\label{sec:experiments}Experiments}
\textbf{Datasets.}
We validate the effectiveness of Knockoffs-SPR on synthetic noisy datasets CIFAR-10 and CIFAR-100~\cite{krizhevsky2009learning}, and real-world noisy datasets  WebVision~\cite{li2017webvision} and Clothing1M~\cite{xiao2015learning}.
We consider two types of noisy labels for CIFAR:
i) Symmetric noise:
Every class is corrupted uniformly with all other labels;
ii) Asymmetric noise:
Labels are corrupted by similar (in pattern) classes. 
WebVision has 2.4 million images collected from the internet with the same category list as ImageNet ILSVRC12.
Clothing1M has 1 million images collected from the internet and labeled by the surrounding texts.

\textbf{Backbones}. 
For CIFAR, we use ResNet-18~\cite{he2016deep} as our backbone. 
For WebVision we use Inception-ResNet~\cite{szegedy2017inception} to extract features to follow previous works. For Clothing1M we use ResNet-50 as backbone. 
For CIFAR and WebVision, we respectively self-supervised pretrain for 100 epochs and 350 epochs using SimSiam~\cite{chen2021exploring}. For Clothing1M, we use ImageNet pre-trained weights to follow previous works.

\textbf{Hyperparameter setting}.
We use SGD to train all the networks with a momentum of 0.9 and a cosine learning rate decay strategy.
The initial learning rate is set as 0.01.
The weight decay is set as 1e-4 for Clothing1M, and 5e-4 for other datasets.
We use a batch size of 128 for all experiments.
We use random crop and random horizontal flip as augmentation strategies.
The network is trained for 180 epochs for CIFAR, 300 epochs for WebVision, and 5 epochs for Clothing1M.
Network training strategy is selected with $p=0.5$ (line 6 in Algorithm~\ref{alg:pipeline}) for Clothing1M, while for other datasets we only use CutMix training.
For features used in Knockoffs-SPR, we reduce the dimension of $\bm{X}$ to the number of classes. For Clothing1M, this is 14 and for other datasets the reduced dimension is 10 (each piece of CIFAR-100 and WebVision contains 10 classes).
We also run SPR with our new network training algorithm (Algorithm~\ref{alg:pipeline}) and report the corresponding results.

\subsection{Evaluation on Synthetic Label Noise}
\textbf{Competitors}. 
We use cross-entropy loss (Standard) as the baseline algorithm for two datasets.
We compare Knockoffs-SPR with algorithms that include
Forgetting~\cite{arpit2017closer} with train the network using dropout strategy, 
Bootstrap~\cite{reed2015training} which trains with bootstrapping, 
Forward Correction~\cite{patrini2017making} which corrects the loss function to get a robust model,
Decoupling~\cite{malach2017decoupling} which uses a meta-update strategy to decouple the update time and update method, 
MentorNet~\cite{jiang2018mentornet} which uses a teacher network to help train the network, 
Co-teaching~\cite{han2018co} which uses two networks to teach each other, 
Co-teaching+~\cite{yu2019does} which further uses an update by disagreement strategy to improve Co-teaching, 
IterNLD~\cite{wang2018iterative} which uses an iterative update strategy, 
RoG~\cite{lee2019robust} which uses generated classifiers, 
PENCIL~\cite{yi2019probabilistic} which uses a probabilistic noise correction strategy, 
GCE~\cite{zhang2018generalized} and SL~\cite{wang2019symmetric} which are extensions of the standard cross-entropy loss function, 
and TopoFilter~\cite{wu2020topological} which uses feature representation to detect noisy data.
For each dataset, all the experiments are run with the same backbone to make a fair comparison.
We run all the experiments with  randomly generated noisy labels for five times and calculate the average and standard deviation of the accuracy of the last epoch.
The results of competitors are reported in~\cite{wu2020topological}.

As in Table~\ref{tab:main-results},
Knockoffs-SPR enjoys a higher performance compared with other competitors on CIFAR, validating the effectiveness of Knockoffs-SPR on different noise scenarios.
SPR enjoys better performance on higher symmetric noise rate of CIFAR-100.
This may contributes to the manual selection threshold of 50\% of the data.
Then SPR will select more data than Knockoffs-SPR, for example in Sym. 80\% noise scenario SPR will select 24816 clean data while Knockoffs-SPR will select 18185.
This leads to a better recovery of clean data (recall of 94.22\% while Knockoffs-SPR is 81.20\%) and thus a better recognition capacity.

\subsection{Evaluation on Real-World Noisy Datasets} 

\begin{table}
\caption{Test accuracies(\%) on WebVision and ILSVRC12.}
\label{tab:webvision}
\centering
\begin{tabular}{lcccc}
\toprule
\multirow{2}{*}{Method}&
\multicolumn{2}{c}{WebVision} &
\multicolumn{2}{c}{WebVision $\rightarrow$ ILSVRC12} \tabularnewline
& top1 & top5 & top1 & top5 \tabularnewline
\midrule
F-correction & 61.12 & 82.68 & 57.36 & 82.36 \tabularnewline
Decoupling & 62.54 & 84.74 & 58.26 & 82.26 \tabularnewline
D2L & 62.68 & 84.00 & 57.80 & 81.36 \tabularnewline
MentorNet & 63.00 & 81.40 & 57.80 & 79.92 \tabularnewline
Co-teaching & 63.58 & 85.20 & 61.48 & 84.70 \tabularnewline
Iterative-CV & 65.24 & 85.34 & 61.60 & 84.98 \tabularnewline
DivideMix & 77.32 & 91.64 & \textbf{75.20} & 90.84 \tabularnewline
\midrule
SPR & 77.08 & 91.40 & 72.32 & 90.92\tabularnewline
Knockoffs-SPR & \textbf{77.96} & \textbf{92.28} & 74.72 & \textbf{92.88}\tabularnewline
\bottomrule
\end{tabular}
\end{table}
\begin{table}
\caption{Test accuracies(\%) on Clothing1M.}
\label{tab:clothing1m}
\centering
\begin{tabular}{lc}
\toprule
Method & Accuracy \tabularnewline
\midrule
Cross-Entropy & 69.21 \tabularnewline
F-correction & 69.84 \tabularnewline
M-correction & 71.00 \tabularnewline
Joint-Optim & 72.16 \tabularnewline
Meta-Cleaner & 72.50 \tabularnewline
Meta-Learning & 73.47 \tabularnewline
P-correction & 73.49 \tabularnewline
TopoFiler & 74.10 \tabularnewline
DivideMix & 74.76 \tabularnewline
\midrule
SPR & 71.16 \tabularnewline
Knockoffs-SPR & \textbf{75.20}\tabularnewline
\bottomrule
\end{tabular}
\end{table}
In this part, we compare Knockoffs-SPR with other methods on real-world noisy datasets: WebVision and Clothing1M. 
We follow previous works to train and test on the first 50 classes of WebVision.
We also evaluate models trained on WebVision to ILSVRC12 to test the cross-dataset accuracy.

\textbf{Competitors}. 
For WebVision, we compare with CE that trains with cross-entropy loss (CE), as well as Decoupling~\cite{malach2017decoupling}, D2L~\cite{ma2018dimensionality}, MentorNet~\cite{jiang2018mentornet}, Co-teaching~\cite{han2018co}, Iterative-CV~\cite{chen2019understanding}, and DivideMix~\cite{Li2020DivideMix}.
For clothing1M, we compare with F-correction~\cite{patrini2017making}, M-correction~\cite{arazo2019unsupervised}, Joint-Optim~\cite{tanaka2018joint}, Meta-Cleaner~\cite{zhang2019metacleaner}, Meta-Learning~\cite{li2019learning}, P-correction~\cite{yi2019probabilistic}, TopoFilter~\cite{wu2020topological} and DivideMix~\cite{Li2020DivideMix}.

The results of real-world datasets are shown in Table~\ref{tab:webvision} and Table~\ref{tab:clothing1m}, where the results of competitors are reported in~\cite{Li2020DivideMix}.
Our algorithm Knockoffs-SPR enjoys superior performance to almost all the competitors, showing the ability of handling real-world challenges.
Compared with SPR, Knockoffs-SPR also achieves better performance, indicating the beneficial of FSR control in real-world problems of learning with noisy labels.

\begin{figure}
\centering
\includegraphics[width=1\linewidth]{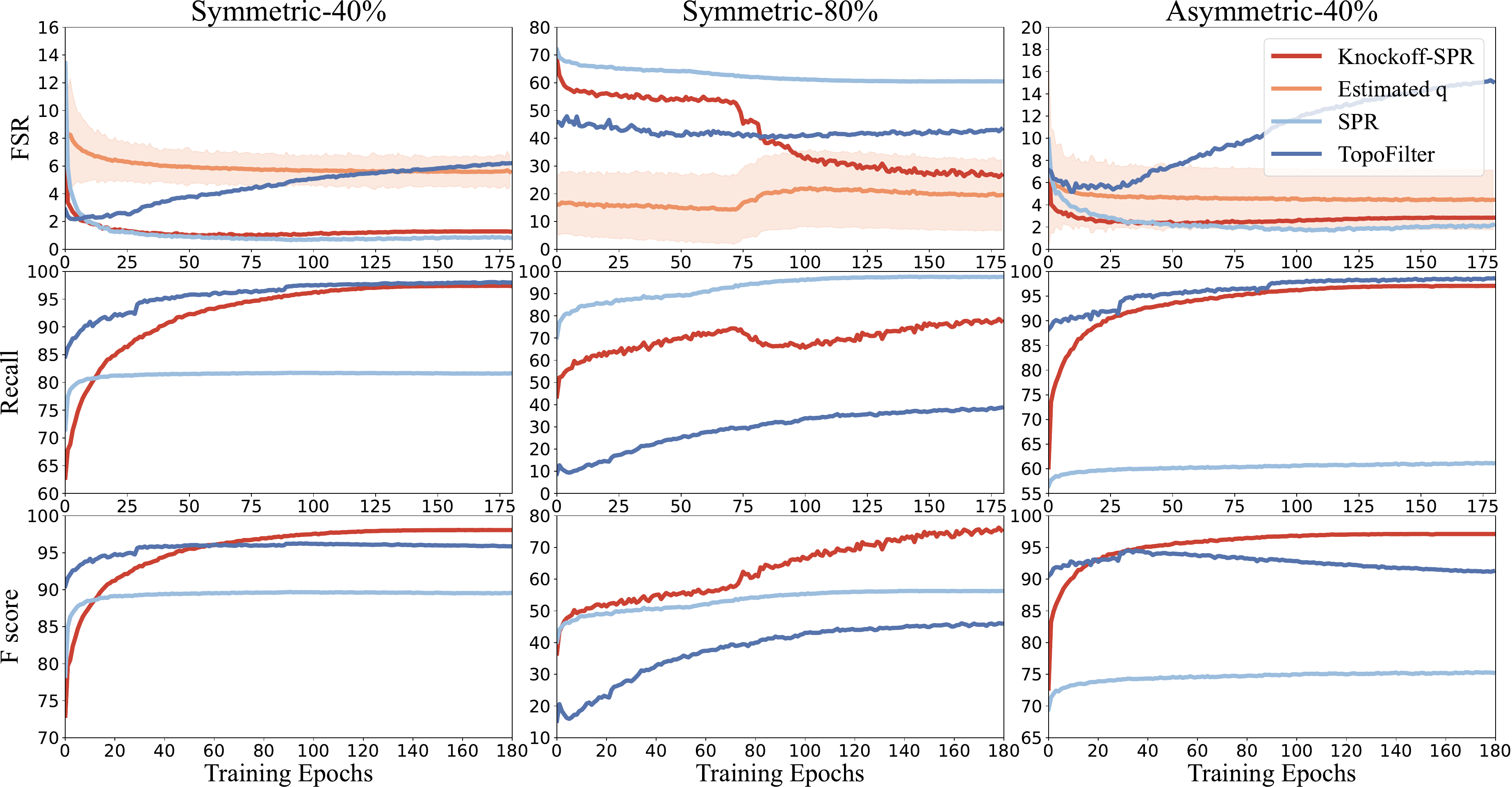}
\caption{Performance(\%) comparison on sample selection along the training path on CIFAR 10 with different noise scenarios.
In the FSR, we also visualize the estimated FSR ($q$) by Knockoffs-SPR, which is the threshold we use to select clean data.
}
\label{fig:FSR-q}
\end{figure}

\subsection{\label{sec:selection-quality}Evaluation of Sample Selection Quality}
To test whether Knockoffs-SPR leads to better sample selection quality, we test the following statistics on CIFAR-10 with different noise scenarios, including Sym. 40\%, Sym. 80\%, and Asy. 40\%.
\textbf{i)} \emph{FSR}: the ratio of falsely selected noisy data in the estimated clean data, which is the target that Knockoffs-SPR aims to control;
\textbf{ii)} \emph{Recall}: the ratio of selected ground-truth clean data in the full ground-truth clean data, which indicates the power of sample selection algorithms;
\textbf{iii)} \emph{F1-score}: the harmonic mean of precision (1-FSR) and recall, which measures the balanced performance of FSR control and power.
We plot the corresponding statistics of each algorithm along the training epochs in Fig.~\ref{fig:FSR-q}.
We further visualize the estimated FSR, $q$, of Knockoffs-SPR to compare with the ground-truth FSR. As we use the splitting algorithm, where each piece contains 10 classes with each class containing a subset of data, we estimate FSR for each piece and report their average and standard deviation.

\textbf{FSR control in practice}.
\textbf{i)} When the noise rate is not high, for example in Sym. 40\% and Asy. 40\% scenarios, the ground-truth FSR is well upper-bounded by the estimated FSR (with no larger than a single standard deviation).
When the noise rate is high, for example in Sym. 80\% noise scenario, the FSR cannot get controlled in the early stage.
However, as the training goes on, FSR can be well-bounded by Knockoffs-SPR, indicating the cycle between
sample selection and network training evolves properly.

\textbf{ii)} When the training set is not very noisy, for example in Sym. 40\% scenario, the true FSR is far below the estimated $q$. This gap can be explained by a good estimation of $\bm{\beta}$ due to the small noisy rate. When $\hat{\bm{\beta}}_1$ can accurately estimate $\bm{\beta}^*$, the $\tilde{\bm{\gamma}}^{*}_{2,j}$ dominate in Eq.~\eqref{eq:kkt}. Therefore, the $P(W_j>0|\tilde{\bm{y}}_{2,j} \text{ is clean}) > \frac{1}{2}$, making $P(W_j > 0) > 1/2> \frac{c-2}{2(c-1)}$. Since the true FSR bound is inversely proportional to $P(W_j > 0)$ (FSR $\propto \max_{j \in \mathcal{C}^c} 1/P(W_j>0)-1$), it is smaller than the theoretical bound $q$. 

\textbf{Sample selection quality comparison.}
We compare the sample selection quality of Knockoffs-SPR with SPR and TopoFilter~\cite{wu2020topological}.
\textbf{i)} Knockoffs-SPR enjoys the (almost) best FSR control capacity in all noise scenarios, especially in the high noise rate setting.
Other algorithms can suffer from failure in controlling the FSR (for example in Sym. 80\% scenario).
\textbf{ii)} The power of Knockoffs-SPR is comparable to the best algorithms in Sym. 40\% and Asy. 40\% scenarios.
For the Sym. 80\% case, Knockoffs-SPR sacrifices some power for FSR control.
\textbf{iii)} Knockoffs-SPR enjoys the best F1 score on sample selection quality, which well-establishes its superiority in selecting  clean data with FSR control.

Note that the SPR algorithm enjoys remarkable theoretical properties, allowing us to identify all noisy data under certain conditions. However, a major drawback lies in our inability to access in advance whether these conditions will hold, as well as our inability to determine an appropriate selection ratio due to the unknown noise rate of the training dataset.
In practical applications, we typically set the selection ratio of SPR to 50\%. 
When the actual noise rate is close to this ratio, we observe excellent selection performance and effective FSR control, suggesting that the theoretical conditions are met in most cases.
Conversely, when the noise rate significantly deviates from the selection ratio, SPR may result in a high FSR. 
In such scenarios, our newly proposed approach, Knockoffs-SPR, demonstrates a significant improvement in FSR control, as validated in the Sym.-80\% setting illustrated in Fig.~\ref{fig:FSR-q}.

\subsection{\label{sec:ablation}Further Analysis}

\begin{table}
\caption{Ablation(\%) of Knockoffs-SPR on CIFAR-10.}
\label{tab:pp-ablation}
\centering
\begin{tabular}{llccc}
\toprule
Setting & Method & Acc. & FSR & q \tabularnewline
\midrule
\midrule
\multirow{6}{*}{Sym. 40\%} & 
SPR & 94.0 & 0.82 & -\tabularnewline
& $*$-random & 92.0 & 23.04 & 4.31\ci{0.73} \tabularnewline
& $*$-multi & 94.4 & 1.31 & 2.00\ci{0.00} \tabularnewline
& $*$-noPCA  & 81.7 & 11.51 & 14.18\ci{7.62} \tabularnewline
& $*$-NN & 94.3 & 1.28 &5.62\ci{1.07}\tabularnewline
& Knockoffs-SPR  & 94.7 & 1.27 & 5.59\ci{1.11}\tabularnewline
\midrule
\multirow{6}{*}{Sym. 80\%} & 
SPR & 78.0 & 60.47 & -  
\tabularnewline
& $*$-random & 84.6 & 49.76 & 9.47\ci{4.39} \tabularnewline
& $*$-multi & 83.0 & 25.77 & 2.22\ci{0.62} \tabularnewline
& $*$-noPCA & 10.0  & 78.06 & 25.95\ci{11.88} \tabularnewline
& $*$-NN & 79.3 & 29.70 & 19.67\ci{12.79}\tabularnewline
& Knockoffs-SPR & 84.3 & 26.72 & 19.52\ci{12.77} \tabularnewline
\midrule
\multirow{6}{*}{Asy. 40\%} 
& SPR  & 89.5 & 2.19 & -
\tabularnewline
& $*$-random & 84.4 & 16.94 & 4.15\ci{2.59} \tabularnewline
& $*$-multi & 93.4  & 2.97 & 2.00\ci{0.00}\tabularnewline
& $*$-noPCA  & 93.7 & 7.62 & 5.22\ci{2.85}\tabularnewline
& $*$-NN & 93.8 & 2.77 &4.44\ci{2.67}\tabularnewline
& Knockoffs-SPR & 93.5 & 2.84 & 4.45\ci{2.68}\tabularnewline
\bottomrule
\end{tabular}
\end{table}

\textbf{Influence of Knockoffs-SPR strategies}.
We compare Knockoffs-SPR with several variants, including:
SPR (The original SPR algorithm),
$*$-random (Knockoffs-SPR with randomly permuted labels),
$*$-multi (Knockoffs-SPR without class-specific selection),
$*$-noPCA (Knockoffs-SPR without using PCA to pre-process the features) and $*$-NN (Knockoffs-SPR that directly use an additional linear layer to reduce the feature dimension).
Experiments are conducted on CIFAR-10 with different noise scenarios, as in Table~\ref{tab:pp-ablation}. 
We observe the following results: 

\textbf{i)} As also shown in Fig.~\ref{fig:FSR-q}, the SPR can control the FSR in Sym. 40\% and Asy. 40\% but fails in Sym. 80\%. This may be due to that when the noisy pattern is not significant, the collinearity is weak between noisy samples and clean ones, as shown by the distribution of irrepresentable value $\{\Vert (X_\mathcal{S}^\top X_\mathcal{S})^{-1}X_\mathcal{S}^\top X_j \Vert_1 \}_{j \in \mathcal{S}^c}$ in Fig.~1 in the appendix. In this regard, most of the earlier (resp. later) selected samples in the solution path tend to be noisy (resp. clean) samples. When there is strong multi-collinearity and the irrepresentable condition is violated seriously, our proposed Knockoff procedure can help to control the FSR. 
The higher accuracy of Knockoffs-SPR over SPR can be explained by consistent improvements in terms of the F1 score of sample selection capacity, as shown in Fig.~\ref{fig:FSR-q}.

\textbf{ii)} Compared with the random permutation strategy, Knockoffs-SPR with most-confident permutation enjoys much better FSR control and works much better in Sym. 40\% and Asy. 40\% noise scenarios.
In Sym. 80\% noise scenario, the accuracy is comparable, but the most-confident permutation still enjoys much better FSR control. This result empirically demonstrates the advantage of the most-confident permutation over the random permutation.

\textbf{iii)} Running Knockoffs-SPR on each class separately is beneficial for the FSR control capacity and recognition capacity.
When the noise rate is high, for example in Sym. 80\% noise scenario, running Knockoffs-SPR on multiple classes cannot control the FSR properly by $q$.

\textbf{iv)} Using PCA to pre-process features is beneficial for FSR control in all cases and will increase the recognition capacity in some cases, especially when the noise rate is high.

\textbf{v)} 
We can introduce an additional linear layer to reduce feature dimensions within the network.
This approach yields results similar to employing PCA when dealing with moderate noise scenarios. However, it leads to a degradation of both FSR control and recognition capabilities in high noise scenario.
This observation highlights the robustness of PCA as a dimension reduction technique for Knockoffs-SPR. Additionally, it positions Knockoffs-SPR as an easily integrable sample selection module for various learning frameworks with no modification.

\begin{table}
\centering
\caption{Computation efficiency  of the splitting algorithm on CIFAR-10.\label{tab:running-time}}
\begin{tabular}{lc}
\toprule
Model & Training Time for one epoch\tabularnewline
\midrule
\midrule
Knockoffs-SPR w/o split algorithm & about 6h\tabularnewline
Knockoffs-SPR w/ split algorithm & 66s\tabularnewline
\bottomrule
\end{tabular}
\end{table}
\textbf{Influence of \underline{\textit{Scalable}}}.
In our framework, we propose a split algorithm to divide the whole training set into small pieces to run Knockoffs-SPR in parallel.
In this part, we compare the running time between using the split algorithm and not using it.
Results are shown in Table~\ref{tab:running-time}. We can see that the splitting algorithm can significantly reduce the computation time. This is important in large-scale applications.

While it is impractical to apply Knockoffs-SPR to the entire  training data, we conducted an analysis of its capacity on pieces. 

\textbf{i)} The break of majority-clean assumption:
Due to our sampling strategy, we may not guarantee that clean data  will always constitute the majority within each piece. 
However, this deviation does not undermine the integrity of our sample selection process.
(1) As we employ uniformly data sampling for splitting, the pieces violating the majority-clean assumption are only minor fraction. 
Even if our algorithm encounters challenges in these pieces, it effects only a minority of the training data. 
Consequently, the FSR remains under control when considering the expected performance across pieces.
For an illustrative example, refer to the average and standard deviation of $q$ in Fig.~\ref{fig:FSR-q}.
Moreover, since we employ random  the data shuffling during each training epoch, the falsely selected data have limited influence.
(2) As the training progresses, the linear relation between features and ground-truth labels strengthens, evidenced by improved testing  accuracy over time.
Consequently, this linear relation can persist even when clean data is not the majority within a specific piece, which can mitigate the impact of challenging cases.

\textbf{ii)} The FSR control capacity: We assess two key factors the influence the FSR control capacity, the number of classes within each group, and the number of samples within each class.
Regarding the number of classes, note that we have $\mathbb{P}(W_j >0 \mid \gamma_j^*\neq0) = \frac{1}{2} \frac{c-2}{c-1} + \kappa_j \frac{1}{c-1}$. Since $\kappa_j$ can normally be larger than $1/2$ as the noisy label tends to be less credible than the clean label, a smaller $c$ tends to make $W_j < 0$ with less probability. As a result, for any $t$, the number of set $\{j:0<W_j\leq t\}$ tends to be large, while $\{j:-t\leq W_j<0\}$ tends to be small, leads to zero $T$ and thus an empty clean set if $q$ is too small. That means we need a much larger value of $q$ to make the clean set non-empty. Such a larger value of $q$ may result in a higher FSR. 
Additionally, a greater number of samples per class improves the estimation of $\hat{\bm{\beta}}$, benefiting the FSR control capacity of  Knockoffs-SPR.
In general, a higher number of classes and more examples within each class enhance the performance.
To validate these principles, we conduct experiments on CIFAR-10 with different noise scenarios.
To ensure comparability, we use  the same self-supervised pre-trained backbone across all experiments, differing only in the number of classes within each group and number of samples within each piece. 
Results are shown in Fig.~\ref{fig:FSR-split}, clearly demonstrating that greater number of classes and more examples within each piece result in lower FSR.
In practice, we typically set the number of classes within each group to 10 and the number of samples within each class to 75, guided by hardware constraints. These hyperparameters can be increased if more computational resources or more efficient algorithms become available, which we leave as a potential avenue for future research.

\begin{figure}
\centering
\includegraphics[width=1\linewidth]{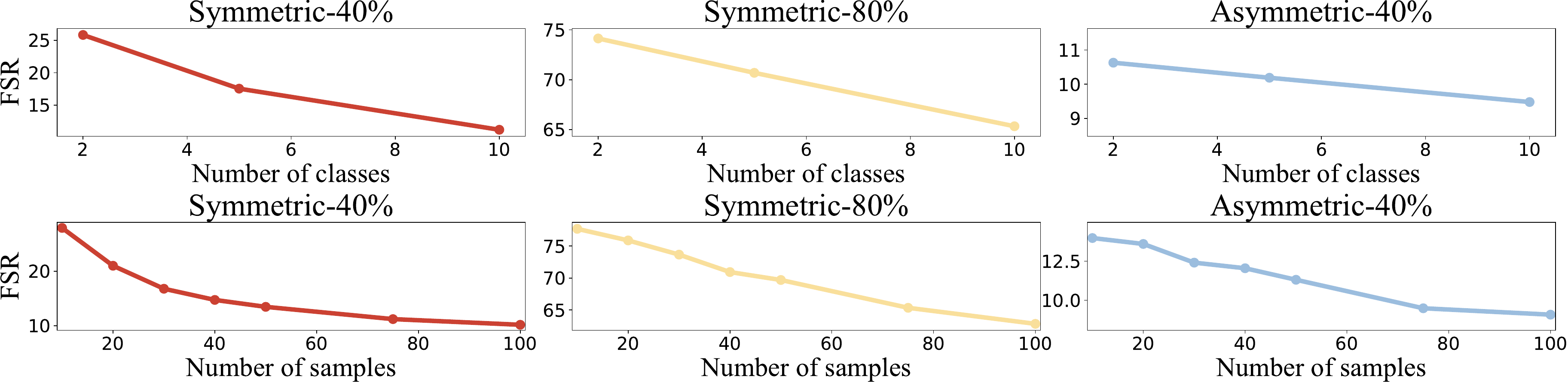}
\caption{FSR(\%) comparison of different number of classes and number of samples in Knockoffs-SPR on CIFAR 10 with different noise scenarios.
}
\label{fig:FSR-split}
\end{figure}

\begin{table}
\caption{Ablation(\%) of training strategies on CIFAR-10.}
\label{tab:network-ablation}
\centering
\begin{tabular}{lccc}
\toprule
Method & Sym. 40\% & Sym. 80\% & Asy. 40\% \tabularnewline
\midrule
Knockoffs-SPR - Self & 92.5 & 24.3 & 92.2\tabularnewline
Knockoffs-SPR - Semi & 91.3 & 54.0 & 88.5\tabularnewline
Knockoffs-SPR - EMA & 94.5 & 83.8 & 93.2\tabularnewline
Knockoffs-SPR & 94.7 & 84.3 & 93.5 \tabularnewline
\bottomrule
\end{tabular}
\end{table}
\textbf{Influence of network training strategies}.
To better train the network, we adopt the self-supervised pre-trained backbone and the semi-supervised learning framework with an EMA update model.
In this part, we test the influence of these strategies on CIFAR-10 with different noise scenarios. Concretely, we compare the full framework with \emph{Knockoffs-SPR - Self} which uses a randomly initialized backbone, \emph{Knockoffs-SPR - Semi} which uses supervised training, and \emph{Knockoffs-SPR - EMA} which does not use the EMA update model. Results are summarized in table.~\ref{tab:network-ablation}. 
We can find that: 
\textbf{i)} The self-supervised pre-training is important for high noise rate scenarios, while for other settings, it is not so essential;
\textbf{ii)} Semi-supervised training consistently improves the recognition capacity, indicating the utility of leveraging the support of noisy data;
\textbf{iii)} The EMA model will slightly improve the recognition capacity.

\begin{figure}
\centering
\includegraphics[width=1\linewidth]{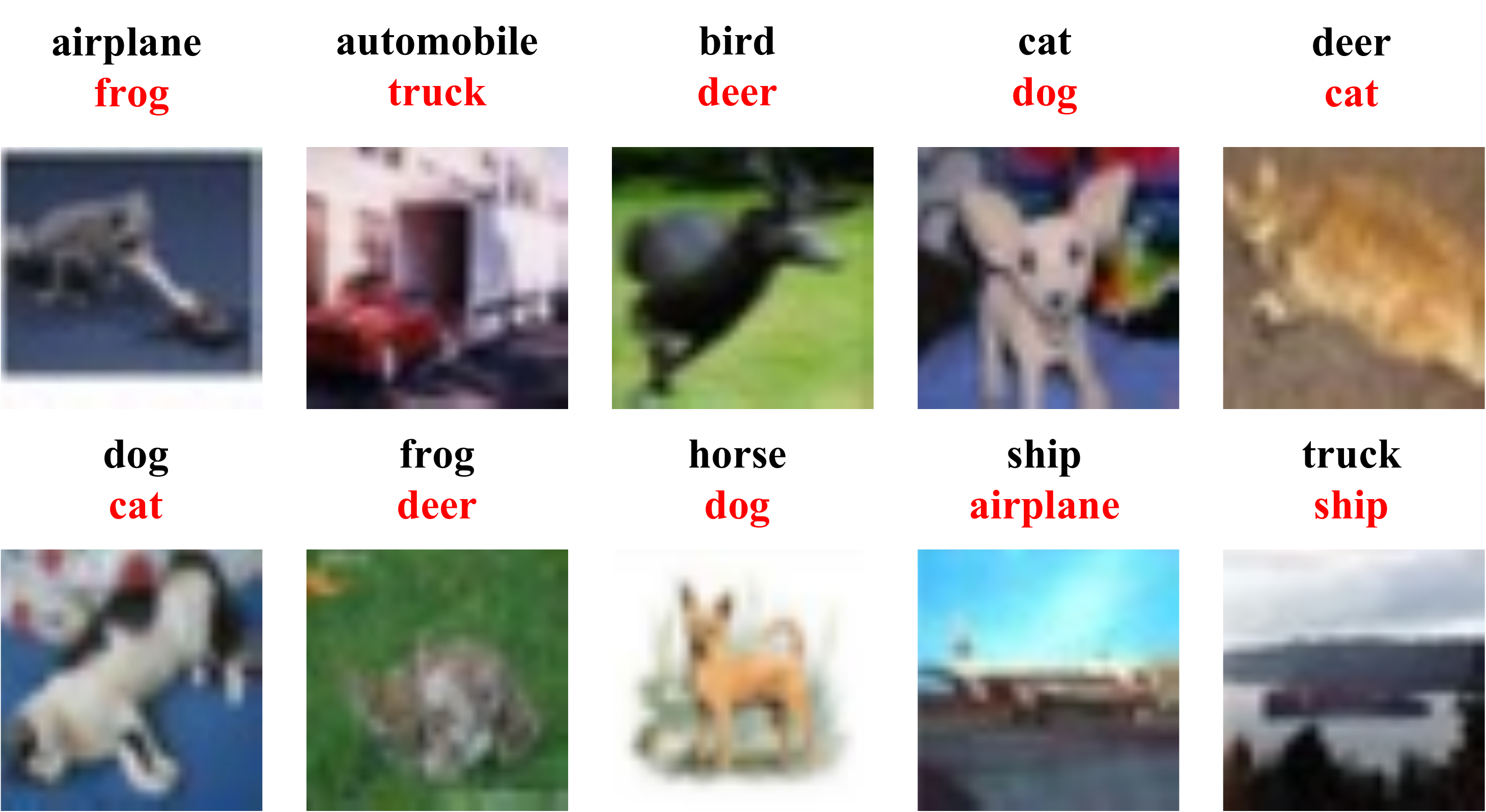}
\caption{Qualitative results of falsely selected examples by Knockoffs-SPR.
The black words are the labeled classes while the real classes are denoted by red words.
}
\label{fig:example}
\end{figure}
\textbf{Qualitative visualization}.
We randomly visualize some falsely selected examples of CIFAR-10 in Fig.~\ref{fig:example}. Most of these cases have some patterns that confuse the noisy label and the true label, thus making Knockoffs-SPR falsely identify them as clean samples.

%% file: chap/conclusion.tex
\section{Conclusion}
\label{sec:conclusion}

This paper proposes a statistical sample selection framework -- Scalable Penalized Regression with Knockoff Filters (Knockoffs-SPR) to select clean data with a controlled false selection rate.
Specifically, we propose an equivalent leave-one-out $t$-test approach as a penalized linear model, in which  zero mean-shift parameters can be induced as an indicator for clean data.
We propose a delicate Knockoffs-SPR algorithm to identify clean samples in a way that the false selection rate is controlled by the user-provided upper bound.
Such an upper bound is proved theoretically and works well in empirical results.
Experiments on several synthetic and  real-world datasets demonstrate the effectiveness of Knockoff-SPR.

\noindent\textbf{
Acknowledgements.
}
This work was supported in part by National Key Research and Development Program of China (No. 2022YFC2405100), the National Natural Science Foundation of China Grant (62076067), and the State Key Program of National Natural Science Foundation of China under Grant No. 12331009.

%% file: appendix.tex
In this supplementary material, we formally present the proof of FSR control theorem of knockoff-SPR in Sec.~\ref{sec:FSR}.
For consistency, we also provide the proof of the noisy set recovery theorem of SPR in Sec.~\ref{sec:spr}.
Some additional experimental results are provided in Sec.~\ref{sec:visualization}.

\section{\label{sec:FSR}FSR Control theorem of knockoff-SPR}
Recall that we are solving the problem of
\begin{equation}
\label{eq:spr-permute}
\begin{cases}
\frac{1}{2}\left\Vert \bm{Y}_2-\bm{X}_2\tilde{\bm{\beta}}_1-\bm{\gamma}_2\right\Vert _{\mathrm{F}}^{2}+\sum_j P(\bm{\gamma}_{2,j};\lambda),\\
\frac{1}{2}\left\Vert \tilde{\bm{Y}}_2-\bm{X}_2\tilde{\bm{\beta}}_1-\tilde{\bm{\gamma}}_2\right\Vert _{\mathrm{F}}^{2}+\sum_j P(\tilde{\bm{\gamma}}_{2,j};\lambda),
\end{cases}
\end{equation}
where $\bm{Y}_2,\tilde{\bm{Y}}_2,\bm{\gamma}_2,\tilde{\bm{\gamma}}_2\in\mathbb{R}^{n\times c}$, $\bm{X}_2 \in\mathbb{R}^{n\times p},\hat{\bm{\beta}}_1\in\mathbb{R}^{p\times c}$.
We introduce the following lemma from~\cite{barber2019knockoff} and~\cite{cao2021control}.
\begin{lem}
\label{lem:barber}
Suppose that $B_1,\ldots,B_n$ are indenpendent variables, with $B_i\sim\mathrm{Bernoulli}(\rho_i)$ for each $i$, where $\min_i \rho_i\geq\rho>0$. Let $J$ be a stopping time in reverse time with respect to the filtration $\{\mathcal{F}_j\}$, where
\[
\mathcal{F}_j=\sigma\left(\{B_1+\cdots+B_j,B_{j+1},\ldots,B_n\}\right).
\]
Then
\[ 
\mathbb{E}\left[\frac{1+J}{1+B_1+\cdots+B_J}\right]\leq\rho^{-1}.
\]
\end{lem}
\begin{proof}
We first follow~\cite{barber2019knockoff} to prove the case when $\{B_i\}$ are i.i.d. variables with $B_i\sim\mathrm{Bernoulli}(\rho)$, where $\rho>0$.
Then we follow~\cite{cao2021control} to generalize the conclusion to non-identical case.

Define the stochastic process
\begin{equation}
M_j \coloneqq \frac{1+j}{1+S_j}\quad\textrm{with}\quad S_j\coloneqq B_1+\cdots+B_j
\end{equation}
We show that $\{M_j\}$ is a super-martingale with respect to the reverse filtration $\{\mathcal{F}_j\}$.
It is trivial that $\{M_j\}$ is $\{\mathcal{F}_j\}$-adapted and $\{\mathcal{F}_j\}$ is reverse filtration, that is a decreasing sequence 
\begin{equation}
\mathcal{F}_j\subset\mathcal{F}_{j-1}\cdots\subset \{B_i\}_{i=1}^n
\end{equation}
with each $\mathcal{F}_j$ be a sub-$\sigma$-algebras of $\sigma(\{B_i\}_{i=1}^n)$.
Further, we have
$
\mathbb{E}\left[\left|M_j\right|\right]\leq 1+j\leq 1+n <\infty
$
with fixed $n$.
Now we bound the conditional expectation $\mathbb{E}[M_j\mid\mathcal{F}_{j+1}]$.
Note that since $\{B_j\}_{j=1}^{i+1}$ are i.i.d. variable and thus exchangeable when conditioned on $\mathcal{F}_{j+1}$, then we have
\begin{equation}
\mathbb{P}(B_{j+1}=1\mid\mathcal{F}_{j+1})=\frac{S_{j+1}}{j+1}
\end{equation}
When $S_{j+1}=0$, it is natural that $S_{j}=0$ thus $M_j=1+j < 1+ (j+1)=M_{j+1}$.
When $S_{j+}>0$, we have 
\begin{equation}
\begin{aligned}
\mathbb{E}[M_j\mid\mathcal{F}_{j+1}]
=&\frac{1+j}{1+S_{j+1}-1}\cdot\mathbb{P}(B_{j+1}=1\mid\mathcal{F}_{j+1})\\
&+\frac{1+j}{1+S_{j+1}}\cdot\mathbb{P}(B_{j+1}=0\mid\mathcal{F}_{j+1})\\
=&\frac{1+j}{S_{j+1}}\cdot\frac{S_{j+1}}{j+1}+\frac{1+j}{1+S_{j+1}}\cdot\frac{j+1-S_{j+1}}{j+1}\\
=&\frac{1+(j+1)}{1+S_{j+1}}\\
=&M_{j+1}.
\end{aligned}
\end{equation}
Hence we have $\mathbb{E}[M_j\mid\mathcal{F}_{j+1}]\leq M_{j+1}$, which finishes the proof for the super-martingale $\{M_j\}$.
Then by the Doob's optional sampling theorem~\cite{doob1953stochastic}, we have
\begin{equation}
\mathbb{E}[M_j]\leq\mathbb{E}[M_n].
\end{equation}
Finally, we have
\begin{equation}
\begin{aligned}
\mathbb{E}[M_n]&=\mathbb{E}[\frac{1+n}{1+S_n}]\\
&=(1+n)\sum_{m=0}^n\frac{1}{1+m}\cdot\frac{n!}{m!(n-m)!}\rho^m (1-\rho)^{n-m}\\
&=\rho^{-1}(1 - (1-\rho)^{n+1})\\
&\leq \rho^{-1}.
\end{aligned}
\end{equation}
Now it suffices to show that the conclusion also holds for non-identical Bernoulli variables.
Following~\cite{cao2021control}, for each $B_i\sim\mathrm{Bernoulli}(\rho_i)$, we construct the following disjoint Borel sets
$\{A_{j}^{i}\}_{j=1}^4$ such that $\mathbb{R}=\cup_{j=1}^4A_j$ with
\begin{equation}
\begin{aligned}
&\mathbb{P}(A_1^i)=1-\rho_i;\quad
\mathbb{P}(A_2^i)=\rho\frac{1-\rho_i}{1-\rho};\\
&\mathbb{P}(A_3^i)=\rho\frac{\rho_i-\rho}{1-\rho};\quad
\mathbb{P}(A_4^i)=\rho_i-\rho.
\end{aligned}
\end{equation}
Define $U_i=A_1^i\cup A_2^i,V_i=A_2^i\cup A_3^i,G_i=A_2^i\cup A_3^i\cup A_4^i$.
Based on the specific construction we can set $G_i=B_i$.
Further define $Q_i=1\{\xi_i\in V_i\}$ and a random set $A=\{i:\xi_i\in U_i\}$.
Then we have
\begin{equation}
\begin{aligned}
&Q_i\cdot 1\{i\in A\} + 1\{i\notin A\}\\
&=1\{\{\xi_i\in V_i \cap U_i\}\cup\{\xi_i\in U_i^C\}\}\\
&=1\{\{\xi_i\in A_2^i\}\cup\{\xi_i\in A_3^i\cup A_4^i\}\}\\
&=1\{\{\xi_i\in A_2^i\cup A_3^i\cup A_4^i\}\}=B_i.
\end{aligned}
\end{equation}
Hence 
\begin{equation}
\begin{aligned}
\frac{1+j}{1+S_j}&=\frac{1+|i\leq j:i\in A| + |i\leq j:i\notin A|}{1+\sum_{i\leq j,i\in A}Q_i+|i\leq j:i\notin A|}\\
&\leq\frac{1+|i\leq j:i\in A|}{1+\sum_{i\leq j,i\in A}Q_i}.
\end{aligned}
\end{equation}
The inequality holds because $\frac{a+c}{b+c}\leq\frac{a}{b}$ for $0<b\leq a,c\geq 0$.
Note that by definition
\begin{equation}
\begin{aligned}
\mathbb{P}(Q_i=1\mid i\in A)&=\mathbb{P}(\xi_i\in V_i\mid \xi_i \in U_i)\\
&=\frac{\mathbb{P}(A_2^i )}{\mathbb{P}(A_1^i \cup A_2^i)}\\
&=\rho=\mathbb{P}(Q_i=1),\\
\mathbb{P}(Q_i=1\mid i\not\in A)&=\mathbb{P}(\xi_i\in V_i\mid \xi_i \notin U_i)\\
&=\frac{\mathbb{P}(A_3^i )}{\mathbb{P}(A_3^i \cup A_4^i)}\\
&=\rho=\mathbb{P}(Q_i=1).
\end{aligned}
\end{equation}
indicating that $Q_i$ and $A$ are independent.

For any fixed $A$, define $\tilde{Q}_i=Q_i\cdot 1\{i\in A\}$ and the reverse filtration
$\tilde{F}_j=\sigma(\{\sum_{k=1}^j\tilde{Q}_k, \tilde{Q}_{j+1},\ldots,\tilde{Q}_n,A\})$.
Then when conditioned on $A$, the established result suggests that 
\begin{equation}
\mathbb{E}\left[\frac{1+|i\leq j:i\in A|}{1+\sum_{i\leq j,i\in A}Q_i}\bigg| A\right]\leq \rho^{-1}.
\end{equation}
Take expectation over $A$ finishes the proof.
\end{proof}
\subsection{Proof of Theorem 2}
\begin{proof}
We first control the FSR rate of the second subset.
Specifically, we have
\begin{equation}
\begin{aligned}
\mathrm{FSR}(T) &\leq\mathbb{E}\left[\frac{\#\left\{ j:\bm{\gamma}_{j}\neq0\textrm{ and }-T\leq W_j<0\right\}}{1+\#\left\{ j:\bm{\gamma}_{j}\neq0\textrm{ and }0< W_j\leq T\right\}}\right.  \\
& \left. \cdot \frac{1+\#\left\{ j:0< W_j\leq T\right\}}{\#\left\{ j:-T\leq W_j<0\right\} \lor1}\right] \\
& \leq q\cdot \mathbb{E}\left[\frac{\#\left\{ j:\bm{\gamma}_{j}\neq0\textrm{ and }-T\leq W_j<0\right\}}{1+\#\left\{ j:\bm{\gamma}_{j}\neq0\textrm{ and }0< W_j\leq T\right\}} \right]. 
\end{aligned}
\end{equation}
The second inequality holds by the definition of $T$.
Now it suffices to 
bound
\begin{equation}
\mathbb{E}\left[\frac{\#\left\{ j:\bm{\gamma}_{j}\neq0\textrm{ and }-T\leq W_j<0\right\}}{1+\#\left\{ j:\bm{\gamma}_{j}\neq0\textrm{ and }0< W_j\leq T\right\}} \right].
\end{equation}
Given $\bm{\gamma}_{j}\neq0$, consider the following decomposition,
\begin{equation}
\begin{split}
\mathbb{P}(W_j > 0) 
=& \mathbb{P}(W_j>0 | \tilde{\bm{\gamma}}_j^*\neq0) \mathbb{P}(\tilde{\bm{\gamma}}_j^*\neq0) \\
&+ \mathbb{P}(W_j>0 | \tilde{\bm{\gamma}}_j^*=0 ) \mathbb{P}(\tilde{\bm{\gamma}}_j^*=0)
\end{split}
\end{equation}
For $\tilde{\bm{\gamma}}^*_j$, we have a probability of $\frac{1}{c-1}$ to get a clean $\tilde{\bm{\gamma}}_j^*$, 
and a probability of $\frac{c-2}{c-1}$ to get a noisy $\tilde{\bm{\gamma}}_j^*$. 
For $\tilde{\bm{\gamma}}_j^*\neq 0$, we have no information and hence assume a equal probability of $Z_{j}>Z_{j+n}$ and $Z_{j}<Z_{j+n}$.
Then it suffices to bound $\mathbb{P}(W_j>0 | \tilde{\bm{\gamma}}_j^*=0 ) $.
Consider the KKT condition of our problem~\eqref{eq:spr-permute}, and particularly for $\gamma_j^*\neq 0$ and $\tilde{\gamma}_j^*=0$, we have 
\begin{subequations}
\label{eq:kkt}
\begin{align}
\bm{\gamma}_{j}+\frac{\partial P(\bm{\gamma}_{j};\lambda)}{\partial\bm{\gamma}_{j}}=\bm{x}_{j}^{\top}(\bm{\beta}^*-\hat{\bm{\beta}}_1)+\bm{\gamma}_{j}^{*}+\bm{\varepsilon}_{j},\label{eq:kkt-1} \\
\tilde{\bm{\gamma}}_{j}+\frac{\partial P(\tilde{\bm{\gamma}}_{j};\lambda)}{\partial\tilde{\bm{\gamma}}_{j}}=\bm{x}_{j}^{\top}(\bm{\beta}^{*}-\hat{\bm{\beta}}_1)+\tilde{\bm{\gamma}}^{*}_{j}+\tilde{\bm{\varepsilon}}_{j}\label{eq:kkt-2},
\end{align}
\end{subequations}
And
\begin{equation}
\begin{split}
W_j > 0 \iff &\Vert\gamma_j\Vert > \Vert\tilde{\gamma}_j\Vert \\
 \iff &\Vert \bm{x}_{j}^{\top}(\bm{\beta}^*-\hat{\bm{\beta}}_1)+\bm{\gamma}_{j}^{*}+\bm{\varepsilon}_{j}\Vert \\&> \Vert \bm{x}_{j}^{\top}(\bm{\beta}^*-\hat{\bm{\beta}}_1)+\tilde{\bm{\gamma}}^{*}_{j}+\tilde{\bm{\varepsilon}}_{j} \Vert
\end{split}
\end{equation}
Given that $\tilde{\gamma}_j^*=0$ and denote $\bm{a}_j=\bm{x}_{j}^{\top}(\bm{\beta}^*-\hat{\bm{\beta}}_1)$.
Then 
\begin{equation}
\begin{aligned}
&\mathbb{P}(W_j>0 | \tilde{\bm{\gamma}}_j^*=0 )=
\mathbb{P}(\Vert \bm{a}_j +\bm{\gamma}_{j}^{*}+\bm{\varepsilon}_{j}\Vert > \Vert \bm{a}_j +\tilde{\bm{\varepsilon}}_{j}\Vert)\\
&\geq \mathbb{P}(\Vert \bm{a}_j +\bm{\gamma}_{j}^{*}+\bm{\varepsilon}_{j}\Vert > \Vert \bm{a}_j +\tilde{\bm{\varepsilon}}_{j}\Vert \mid \Vert\bm{\varepsilon}_{j}\Vert > M, \Vert\tilde{\bm{\varepsilon}}_{j}\Vert < m)\\
&\cdot \mathbb{P}(\Vert\bm{\varepsilon}_{j}\Vert > M, \Vert\tilde{\bm{\varepsilon}}_{j}\Vert < m)
\end{aligned}
\end{equation}
For Gaussian random variables $\bm{\varepsilon}_{j},\tilde{\bm{\varepsilon}}_{j}$, there exist a sufficiently large $M>0$ and a sufficiently small $m>0$, such that the inequality holds in non-zero probability, and the second probability is also non-zero.
Then we have $\mathbb{P}(W_j>0 | \tilde{\bm{\gamma}}_j^*=0 ) >0$.
\\
On the other hand, we have
\begin{equation}
\begin{aligned}
&1 - \mathbb{P}(W_j>0 | \tilde{\bm{\gamma}}_j^*=0 )=
\mathbb{P}(\Vert \bm{a}_j +\bm{\gamma}_{j}^{*}+\bm{\varepsilon}_{j}\Vert \leq \Vert \bm{a}_j +\tilde{\bm{\varepsilon}}_{j}\Vert)\\
&\geq \mathbb{P}(\Vert \bm{a}_j +\bm{\gamma}_{j}^{*}+\bm{\varepsilon}_{j}\Vert \leq \Vert \bm{a}_j +\tilde{\bm{\varepsilon}}_{j}\Vert \mid \Vert\bm{\varepsilon}_{j}\Vert < m, \Vert\tilde{\bm{\varepsilon}}_{j}\Vert > M)\\
&\cdot \mathbb{P}(\Vert\bm{\varepsilon}_{j}\Vert < m, \Vert\tilde{\bm{\varepsilon}}_{j}\Vert > M).
\end{aligned}
\end{equation}
Similarly, there exist a sufficiently large $M>0$ and a sufficiently small $m>0$ to ensure the non-zero probability on the right, such that this probability is also non-zero.
Denote $\kappa_j\coloneqq \mathbb{P}(W_j>0 | \tilde{\bm{\gamma}}_j^*=0 )$, we have $\kappa_j\in (0,1)$, and thus,
\begin{equation}
\begin{split}
\mathbb{P}(W_j > 0) 
=& \mathbb{P}(W_j>0 | \tilde{\bm{\gamma}}_j^*\neq0) \mathbb{P}(\tilde{\bm{\gamma}}_j^*\neq0) \\&+ \mathbb{P}(W_j>0 | \tilde{\bm{\gamma}}_j^*=0 ) \mathbb{P}(\tilde{\bm{\gamma}}_j^*=0) \\
\geq& \frac{1}{2} \times \frac{c-2}{c-1} + \kappa_j \times \frac{1}{c-1} 
= \frac{c-2+2\kappa_j}{2(c-1)}.
\end{split}
\end{equation}
Hence the random variable $B_j\coloneqq1_{\{W_j>0\}}\sim\mathrm{Bernoulli}(\rho_j)$ for $\bm{\gamma}_{j}\neq0$ with $\rho_j\geq(c-2+2\kappa_{\min})/(2(c-1))$ where $\kappa_{\min}=\min\{\kappa_j\}_{j=1}^n$. 

Now we consider all the $W_j$ of non-null variables, and assumes $|W_1|\leq\cdots\leq|W_n|$ with the abuse of subscripts. 
We have 
\[
\gamma_{j}\neq0\textrm{ and }-T\leq W_j<0 \quad\iff\quad  j\leq J\textrm{ and } B_j = 0
\]
and
\[
\gamma_{j}\neq0\textrm{ and }0< W_j\leq T\quad\iff\quad  j\leq J\textrm{ and } B_j = 1
\]
Hence
\begin{equation}
\begin{aligned}
&\frac{\#\left\{ j:\gamma_{j}\neq0\textrm{ and }-T\leq W_j<0\right\}}{1+\#\left\{ j:\gamma_{j}\neq0\textrm{ and }0< W_j\leq T\right\}}\\
&=\frac{(1-B_1)+\cdots+(1-B_J)}{1+B_1+\cdots+B_J}\\
&=\frac{1+J}{1+B_1+\cdots+B_J}-1.
\end{aligned}
\end{equation}
If we can use Lemma~\ref{lem:barber}, then
\begin{equation}
\begin{split}
&\mathbb{E}\left[\frac{\#\left\{ j:\gamma_{j}\neq0\textrm{ and }-T\leq W_j<0\right\}}{1+\#\left\{ j:\gamma_{j}\neq0\textrm{ and }0< W_j\leq T\right\}}\right]\\
&\leq \rho^{-1}-1
\leq \frac{c-2\kappa_{\min}}{c-2+2\kappa_{\min}}.
\end{split}
\end{equation}
Then we finally get 
\begin{equation}
\mathrm{FSR}(T)\leq q \frac{c-2\kappa_{\min}}{c-2+2\kappa_{\min}}.
\end{equation}
Now it suffices to show that our random variables $\{B_j\}$ are mutually independent.
This is straightforward as we set $P(\bm{\alpha}_{2};\lambda)$ as a sparse penalty for each row $\bm{\alpha}_{2,j}$ in Eq.~\eqref{eq:spr-permute}, respectively.
Then problem of Eq.~\eqref{eq:spr-permute} now is a combination of independent sub-problems for each row $\bm{\alpha}_{2,j}$, and the solution only depends on $(\bm{x}_{2,j},\bm{y}_{2,j},\bm{\beta}(\lambda;\mathcal{D}_1))$.
Then with fixed $\bm{\beta}(\lambda;\mathcal{D}_1)$, the magnitude of $W_j$ is fixed, while the sign of $W_j$ is determined by the permuted label $\tilde{\bm{y}}_{2,j}$, where the mutual independence naturally exist.

Finally, after we control the FSR rate for the second subset, we can get the estimate of $\bm{\beta}(\lambda;\mathcal{D}_2)$ based on the identified clean data in the second subset, and return to run knockoff-SPR on the first subset in a similar pipeline.
Then we have for the whole dataset:
\begin{equation}
\begin{aligned}
\mathrm{FSR}&=\mathbb{E}\left[\frac{|\mathcal{S}_1\cap \mathcal{C}_1| + |\mathcal{S}_2\cap \mathcal{C}_2|}{|\mathcal{C}_1|+ |\mathcal{C}_2|}\right]\\
&\leq\mathbb{E}\left[\frac{|\mathcal{S}_1\cap \mathcal{C}_1|}{|\mathcal{C}_1|}\right]+\mathbb{E}\left[\frac{|\mathcal{S}_2\cap \mathcal{C}_2|}{|\mathcal{C}_2|}\right]\\
&\leq 2\frac{c-2\kappa_{\min}}{c-2+2\kappa_{\min}}q.
\end{aligned}
\end{equation}
To control the FSR with $q$, the threshold of $T$ should be defined as $\frac{c-2+2\kappa_{\min}}{2(c-2\kappa_{\min})}q$, which leads to Theorem 2.
As 
$\frac{c-2\kappa_{\min}}{c-2+2\kappa_{\min}}\leq\frac{c}{c-2}$, we can drop $\kappa_{\min}$ when $c>2$ to get a more elegant result.

\end{proof}

\section{\label{sec:spr}Noisy set recovery theorem of SPR}

Recall that we are solving the problem of 
\begin{equation}
\label{eq:thm-problem}
\min_{\vec{\bm{\gamma}}} \left\Vert\vec{\bm{y}}-\mathring{\bm{X}}\vec{\bm{\gamma}}\right\Vert _{2}^{2}+\lambda \left\Vert \vec{\bm{\gamma}}\right\Vert _{1}.
\end{equation}
\begin{prop}
\label{prop}
Assume that $\mathring{\bm{X}}^\top\mathring{\bm{X}}$ is invertible. 
If 
\begin{equation}
\label{appendix:prop-condition}
\left\Vert\lambda  \mathring{\bm{X}}_{\mathcal{S}^{c}}^{\top} \mathring{\bm{X}}_{\mathcal{S}}\left(\mathring{\bm{X}}_{\mathcal{S}}^{\top} \mathring{\bm{X}}_{\mathcal{S}}\right)^{-1}  \hat{\bm{v}}_{\mathcal{S}}+ \mathring{\bm{X}}_{\mathcal{S}^{c}}^{\top}\left(\bm{I}-\bm{I}_\mathcal{S}\right)(\mathring{\bm{X}} \bm{\varepsilon})\right\Vert_{\infty}<\lambda
\end{equation}
holds for all $\hat{\bm{v}}_{\mathcal{S}} \in[-1,1]^{\mathcal{S}}$, where $\bm{I}_{\mathcal{S}}=\mathring{\bm{X}}_{\mathcal{S}}\left(\mathring{\bm{X}}_{\mathcal{S}}^{\top} \mathring{\bm{X}}_{\mathcal{S}}\right)^{-1} \mathring{\bm{X}}_{\mathcal{S}}^{\top}$,
then the estimator $\hat{\vec{\bm{\gamma}}}$ of Eq.~\eqref{eq:thm-problem} satisfies that
\begin{equation*}
\hat{\mathcal{S}}=\mathrm{supp}\left(\hat{\vec{\bm{\gamma}}}\right)\subseteq \mathrm{supp}\left(\vec{\bm{\gamma}}^*\right)=\mathcal{S}. 
\end{equation*}

\noindent Moreover,
if the sign consistency
\begin{equation}
\label{appendix:sign-consisty}
\operatorname{sign}\left(\hat{\vec{\bm{\gamma}}}_\mathcal{S}\right)=\operatorname{sign}\left(\vec{\bm{\gamma}}^{*}_\mathcal{S}\right)
\end{equation}
holds,
Then 
$\hat{\vec{\bm{\gamma}}}$ is the unique solution of~\eqref{eq:thm-problem} with the same sign as $\hat{\vec{\bm{\gamma}}}^*$.
\end{prop}

\begin{proof}
Note that Eq.~\eqref{eq:thm-problem} is convex that has global minima. 
Denote Eq.~\eqref{eq:thm-problem} as $L$, the solution of $\partial L/\partial\vec{\bm{\gamma}}=0$ is the unique minimizer. 
Hence we have
\begin{equation}
\frac{\partial L}{\partial\vec{\bm{\gamma}}}=-\mathring{\bm{X}}^{\top}\left(\vec{\bm{y}}-\mathring{\bm{X}}\vec{\bm{\gamma}}\right)+\lambda\bm{v}=0
\end{equation}
where $\bm{v}=\partial \left\Vert\vec{\bm{\gamma}}\right\Vert_1 /\partial\vec{\bm{\gamma}}$.
Note that $\left\Vert\vec{\bm{\gamma}}\right\Vert_1 $ is non-differentiable, so we
instead compute its sub-gradient. 
Further note that $v_{i}=\partial\left\Vert \vec{\bm{\gamma}}\right\Vert _{1}/\partial\vec{\gamma}_{i}=\partial\left|\vec{\gamma}_{i}\right|/\partial\gamma_{i}$. 
Hence $v_{i}=\mathrm{sign}\left(\vec{\gamma}_{i}\right)$ if $\vec{\gamma}_{i}\neq0$ and $v_{i}\in\left[-1,1\right]$ if $\vec{\gamma}_{i}=0$.
To distinguish between the two cases, we assume $v_{i}\in\left(-1,1\right)$ if $\vec{\gamma}_{i} = 0$.
Hence there exists $\hat{\bm{v}}\in\mathbb{R}^{n\times1}$ such that
\begin{equation}
-\mathring{\bm{X}}^{\top}\left(\vec{\bm{y}}-\mathring{\bm{X}}  \hat{\vec{\bm{\gamma}}}\right)+\lambda  \hat{\bm{v}}=0,
\end{equation}
$\hat{v}_{i}=\mathrm{sign}\left(\hat{\vec{\gamma}}_{i}\right)$ if $i\in\hat{\mathcal{S}}$ and $\hat{v}_{i}\in(-1,1)$ if $i\in\hat{\mathcal{S}}^c$.

To obtain $\hat{\mathcal{S}} \subseteq \mathcal{S}$, 
we should have $\hat{\vec{\gamma}}_{i}=0$ for $i \in \mathcal{S}^{c}$, that is, $\forall i\in \mathcal{S}^{c},\left|\hat{v}_{i}\right|<1$, i.e.,
\begin{equation}
\left\Vert\mathring{\bm{X}}_{\mathcal{S}^{c}}^{\top}\left(\vec{\bm{y}}-\mathring{\bm{X}}_{\mathcal{S}}  \hat{\vec{\bm{\gamma}}}_\mathcal{S}\right)\right\Vert_{\infty}<\lambda,
\label{eq:prop-inequality}
\end{equation}
For $i \in \mathcal{S}$, we have
\begin{equation}
-\mathring{\bm{X}}_{\mathcal{S}}^{\top}\left(\vec{\bm{y}}-\mathring{\bm{X}}_{\mathcal{S}} \hat{\vec{\bm{\gamma}}}_{\mathcal{S}}\right)+\lambda \hat{\bm{v}}_{\mathcal{S}}=0.
\end{equation}
If $\mathring{\bm{X}}^\top\mathring{\bm{X}}$ is invertible then

\begin{equation}
\hat{\vec{\bm{\gamma}}}_{\mathcal{S}}=\left(\mathring{\bm{X}}_{\mathcal{S}}^{\top}\mathring{\bm{X}}_{\mathcal{S}}\right)^{-1}\left(\mathring{\bm{X}}_{\mathcal{S}}^{\top}\vec{\bm{y}}-\lambda\hat{\bm{v}}_{\mathcal{S}}\right)
\end{equation}
Recall that 
we have
\begin{equation}
\vec{\bm{y}}=\mathring{\bm{X}}_{\mathcal{S}}\vec{\bm{\gamma}}_{\mathcal{S}}^{*}+\mathring{\bm{X}}\vec{\bm{\varepsilon}}
\label{eq:y-ground-truth-appendix}
\end{equation}
Hence
\begin{equation}
\hat{\vec{\bm{\gamma}}}_{\mathcal{S}}=\vec{\bm{\gamma}}_{\mathcal{S}}^{*}+\delta_{\mathcal{S}}, \quad \delta_{\mathcal{S}}:=\left(\mathring{\bm{X}}_{\mathcal{S}}^{\top} \mathring{\bm{X}}_{\mathcal{S}}\right)^{-1}\left[\mathring{\bm{X}}_{\mathcal{S}}^{\top} \mathring{\bm{X}} \vec{\bm{\varepsilon}}-\lambda\hat{\bm{v}}_{\mathcal{S}}\right].
\label{eq:prop-condition}
\end{equation}
Plugging~\eqref{eq:prop-condition} and~\eqref{eq:y-ground-truth-appendix} into~\eqref{eq:prop-inequality} we have
\begin{equation}
\left\Vert\mathring{\bm{X}}_{\mathcal{S}^{c}}^{\top} \mathring{\bm{X}} \vec{\bm{\varepsilon}}- \mathring{\bm{X}}_{\mathcal{S}^{c}}^{\top} \mathring{\bm{X}}_{\mathcal{S}}\left(\mathring{\bm{X}}_{\mathcal{S}}^{\top} \mathring{\bm{X}}_{\mathcal{S}}\right)^{-1}\left[\mathring{\bm{X}}_{\mathcal{S}}^{\top} \mathring{\bm{X}} \vec{\bm{\varepsilon}}-\lambda \hat{\bm{v}}_{\mathcal{S}}\right]\right\Vert_{\infty}<\lambda,
\end{equation}
or equivalently
\begin{equation}
\left\Vert\lambda  \mathring{\bm{X}}_{\mathcal{S}^{c}}^{\top} \mathring{\bm{X}}_{\mathcal{S}}\left(\mathring{\bm{X}}_{\mathcal{S}}^{\top} \mathring{\bm{X}}_{\mathcal{S}}\right)^{-1}  \hat{\bm{v}}_{\mathcal{S}}+ \mathring{\bm{X}}_{\mathcal{S}^{c}}^{\top}\left(\bm{I}-\bm{I}_\mathcal{S}\right)\mathring{\bm{X}} \vec{\bm{\varepsilon}}\right\Vert_{\infty}<\lambda,
\end{equation}
where $\bm{I}_S=\mathring{\bm{X}}_{\mathcal{S}}\left(\mathring{\bm{X}}_{\mathcal{S}}^{\top} \mathring{\bm{X}}_{\mathcal{S}}\right)^{-1} \mathring{\bm{X}}_{\mathcal{S}}^{\top}$.
To ensure the sign consistency, replacing $\hat{\bm{v}}_{\mathcal{S}}=\operatorname{sign}\left(\vec{\bm{\gamma}}_{\mathcal{S}}^{*}\right)$ in the inequality above leads to the final result.
\end{proof}
\begin{lem}
\label{lemma}
Assume that $\vec{\bm{\varepsilon}}$ is 
indenpendent sub-Gaussian with zero mean and bounded variance
$\mathrm{Var}\left(\vec{\bm{\varepsilon}}_{i}\right)\leq\sigma^2$.

Then with probability at least
\begin{equation}
1-2 c n \exp \left(-\frac{\lambda^{2} \eta^{2} }{2  \sigma^{2} \max_{i\in S^{c}}\left\Vert \mathring{\bm{X}}_{i}\right\Vert _{2}^{2}}\right)
\end{equation}
there holds
\begin{equation}
\label{appendix:first-bound}
\left\|\mathring{\bm{X}}_{\mathcal{S}^{c}}^{\top}\left(\bm{I}-\bm{I}_{\mathcal{S}}\right)\left(\mathring{\bm{X}} \vec{\bm{\varepsilon}}\right)\right\|_{\infty} \leq \lambda \eta
\end{equation}
and
\begin{equation}
\label{appendix:second-bound}
\left\|\left(\mathring{\bm{X}}_{\mathcal{S}}^{\top} \mathring{\bm{X}}_{\mathcal{S}}\right)^{-1} \mathring{\bm{X}}_{\mathcal{S}}^{\top} \mathring{\bm{X}} \vec{\bm{\varepsilon}}\right\|_{\infty} \leq \frac{\lambda \eta}{ \sqrt{C_{\min }} \max_{i\in \mathcal{S}^{c}}\left\Vert \mathring{\bm{X}}_{i}\right\Vert _{2}}.
\end{equation}
\end{lem}
\begin{proof}
Let $\bm{z}^c= \mathring{\bm{X}}_{\mathcal{S}^{c}}^{\top}\left(\bm{I}-\bm{I}_{\mathcal{S}}\right)\left(\mathring{\bm{X}} \vec{\bm{\varepsilon}}\right)$, for each  $i\in \mathcal{S}^c$ the variance can be bounded by
\begin{equation*}
\operatorname{Var}\left(\bm{z}_i^c\right) \leq\sigma^{2}\mathring{\bm{X}}_{i}^{\top}\left(\bm{I}-\bm{I}_{S}\right)^{2}\mathring{\bm{X}}_{i} \leq  \sigma^{2} \max_{i\in \mathcal{S}^{c}}\left\Vert \mathring{\bm{X}}_{i}\right\Vert _{2}^{2}.
\end{equation*}
Hoeffding inequality implies that 
\begin{equation*}
\begin{aligned}
& \mathbb{P}\left(\left\| \mathring{\bm{X}}_{\mathcal{S}^{c}}^{\top}\left(\bm{I}-\bm{I}_{\mathcal{S}}\right)\left(\mathring{\bm{X}} \vec{\bm{\varepsilon}}\right)\right\|_{\infty} \geq t\right)\\
& \leq 2\left|\mathcal{S}^{c}\right| \exp \left(-\frac{t^{2} }{2  \sigma^{2} \max_{i\in \mathcal{S}^{c}}\left\Vert \mathring{\bm{X}}_{i}\right\Vert _{2}^{2}}\right),
\end{aligned}
\end{equation*}
Setting $t=\lambda\eta$ leads to the result.

Now let $\bm{z}=\left(\mathring{\bm{X}}_{\mathcal{S}}^{\top}\mathring{\bm{X}}_{\mathcal{S}}\right)^{-1}\mathring{\bm{X}}_{\mathcal{S}}^{\top}\mathring{\bm{X}}\vec{\bm{\varepsilon}}$, we have
\begin{equation*}
\begin{aligned} \mathrm{Var}\left(\bm{z}\right) &=\left(\mathring{\bm{X}}_{\mathcal{S}}^{\top}\mathring{\bm{X}}_{\mathcal{S}}\right)^{-1}\mathring{\bm{X}}_{\mathcal{S}}^{\top}\mathring{\bm{X}}\mathrm{Var}\left(\vec{\bm{\varepsilon}}\right)\mathring{\bm{X}}^{\top}\mathring{\bm{X}}_{\mathcal{S}}\left(\mathring{\bm{X}}_{\mathcal{S}}^{\top}\mathring{\bm{X}}_{\mathcal{S}}\right)^{-1}  \\ 
& \leq\sigma^{2}\left(\mathring{\bm{X}}_{\mathcal{S}}^{\top}\mathring{\bm{X}}_{\mathcal{S}}\right)^{-1}
 \leq\frac{\sigma^{2}}{C_{\min}}\bm{I}.
\end{aligned}
\end{equation*}
Then
\begin{equation*}
\mathbb{P}\left(\left\|\left(\mathring{\bm{X}}_{\mathcal{S}}^{\top} \mathring{\bm{X}}_{\mathcal{S}}\right)^{-1} \mathring{\bm{X}}_{\mathcal{S}}^{\top} \mathring{\bm{X}}  \vec{\bm{\varepsilon}}\right\|_{\infty} \geq t\right) \leq 2\left|S\right| \exp \left(-\frac{t^{2}  C_{\min } }{2 \sigma^{2}}\right).
\end{equation*}
Choose
\begin{equation}
t=\frac{\lambda \eta}{ \sqrt{C_{\min }} \max_{i\in \mathcal{S}^{c}}\left\Vert \mathring{\bm{X}}_{i}\right\Vert _{2}},
\end{equation}
then there holds
\begin{equation*}
\begin{aligned}
& \mathbb{P}\left\{\|\left(\mathring{\bm{X}}_{\mathcal{S}}^{\top} \mathring{\bm{X}}_{\mathcal{S}}\right)^{-1} \mathring{\bm{X}}_{\mathcal{S}}^{\top} \mathring{\bm{X}}  \vec{\bm{\varepsilon}}\|_{\infty} \geq \frac{\lambda \eta}{ \sqrt{C_{\min }} \max_{i\in \mathcal{S}^{c}}\left\Vert \mathring{\bm{X}}_{i}\right\Vert _{2}}\right\} \\
& \leq 2\left|\mathcal{S}\right| \exp \left(-\frac{\lambda^{2} \eta^{2} }{2  \sigma^{2} \max_{i\in \mathcal{S}^{c}}\left\Vert \mathring{\bm{X}}_{i}\right\Vert _{2}^{2}}\right).
\end{aligned}
\end{equation*}
\end{proof}
\subsection{Proof of Theorem 1}
\begin{proof}
The proof essentially follows the treatment in~\cite{wainwright2009sharp}.
The results follow by applying Lemma~\ref{lemma} to Proposition~\ref{prop}.
Inequality~\eqref{appendix:prop-condition} holds if condition C2 and the first bound~\eqref{appendix:first-bound} hold, 
which proves the first part of the theorem.
The sign consistency~\eqref{appendix:sign-consisty} holds if condition C3 and the second bound~\eqref{appendix:second-bound} hold, 
which gives the second part of the theorem.

It suﬃces to show that $\hat{\mathcal{S}}\subseteq \mathcal{S}$ implies $\hat{\mathcal{C}}^c\subseteq \mathcal{C}^c$.
Consider one instance $i$, there are three possible cases for $\bm{\gamma}_{i}^{*}\in\mathbb{R}^{1\times c}$: 
i) $\gamma_{i,j}^{*}\neq0,\forall j\in\left[c\right]$; 
ii) $\gamma_{i,j}^{*}=0,\forall j\in\left[c\right]$; 
iii) $\exists j,k\in\left[c\right],s.t.\ \gamma_{i,j}^{*}=0,\gamma_{i,k}^{*}\neq0$.
If instance $i$ follows case i or case iii, then $i\in \mathcal{C}^c$.
If it follows case ii, then $i\in \mathcal{C}$,
and the indexes of all elements of $\bm{\gamma}_i$ are in $\mathcal{S}^c$.
Since we have $\hat{\mathcal{S}}\subseteq \mathcal{S}$, all elements of $\bm{\gamma}_{i}$ is in $\hat{\mathcal{S}}^{c}$, hence $i\in\hat{\mathcal{C}}$. 
Then we have $\hat{\mathcal{C}}^c\subseteq \mathcal{C}^c$.
\end{proof}

\section{\label{sec:visualization}More Experimental Results}
\textbf{Histogram of the median value of IRR condition of SPR}.
We visualize the median value of the irrepresentable (IRR) value, \emph{i.e.}, $\{\Vert (\mathbf{X}_\mathcal{S}^\top \mathbf{X}_\mathcal{S})^{-1}\mathbf{X}_{\mathcal{S}}^\top X_j \Vert_{1}\}_j$ of SPR final epoch on CIFAR10 with various noisy scenarios in Fig.~\ref{fig:spr-irr}.
As SPR is running on each piece split from the training set,
we calculate matrix $\mathring{\bm{X}}_{\mathcal{S}^c}^{\top}\mathring{\bm{X}}_{\mathcal{S}}(\mathring{\bm{X}}_{\mathcal{S}}^{\top}\mathring{\bm{X}}_{\mathcal{S}})^{-1}$ in irrepresentable condition (C2 in Theorem 1) for each piece at the final epoch.
Then the $L_1$ norm of each row of the matrix is the IRR value of corresponding clean data.
The median value of IRR values in a single piece is used to construct the histogram.
For the noise scenario of Asy. 40\% and Sym. 40\%, the median IRR value is small, indicating weak collinearity between clean data and noisy data.
In these cases, SPR has more chance to distinguish noisy data from clean data and thus leads to a good FSR control capacity.
For the noise scenario of Sym. 80\%, the median IRR values are much larger, indicating a strong multi-collinearity.
Thus SPR can hardly distinguish between clean data and noisy data, leading to a high FSR rate.
\begin{figure}[h]
\centering
\includegraphics[width=1\linewidth]{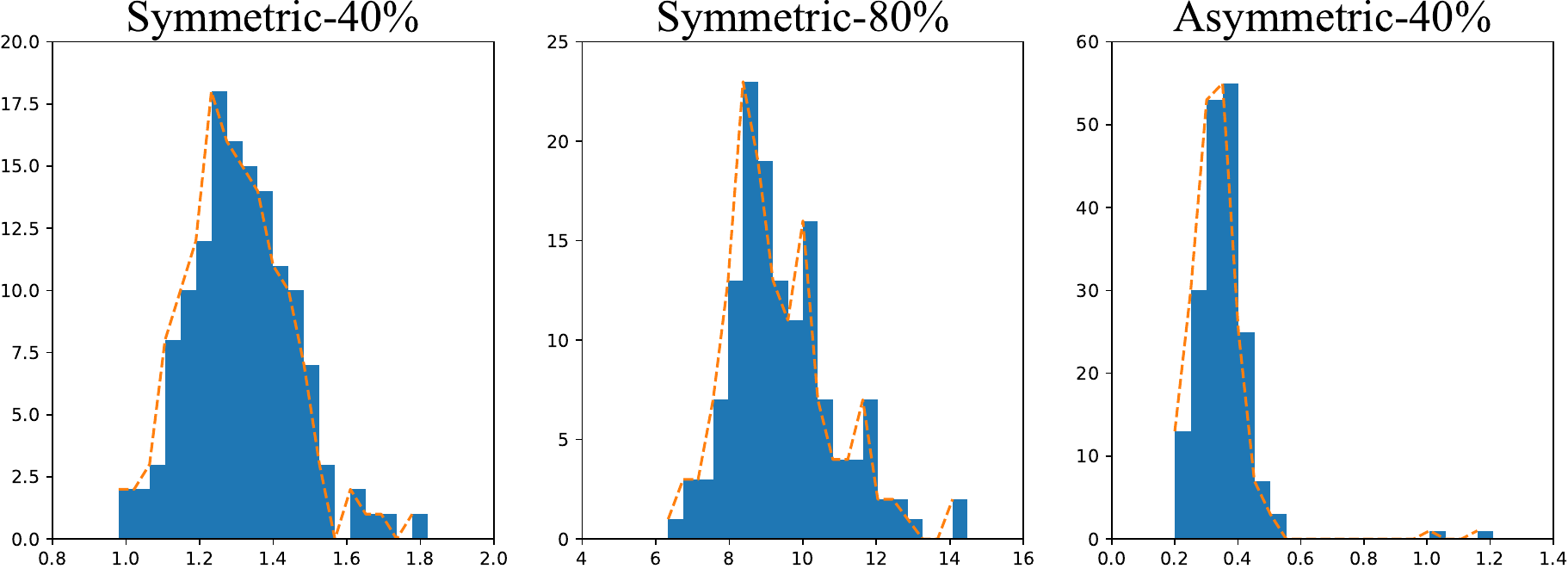}
\caption{Histogram of the median value of the IRR value of SPR
on CIFAR10 with various noisy scenarios.}
\label{fig:spr-irr}
\end{figure}
